\documentclass[10pt,journal,compsoc]{IEEEtran}
\ifCLASSOPTIONcompsoc
  \usepackage[nocompress]{cite}
\else
  \usepackage{cite}
\fi
\ifCLASSINFOpdf
   \usepackage[pdftex]{graphicx}
   \DeclareGraphicsExtensions{.pdf,.jpeg,.png}
\else
   \usepackage[dvips]{graphicx}
   \DeclareGraphicsExtensions{.eps}
\fi
\usepackage{amssymb,amsthm,bbm}
\usepackage[cmex10]{amsmath}
\usepackage[ruled,vlined]{algorithm2e}
\usepackage{algorithmic}
\ifCLASSOPTIONcompsoc
  \usepackage[caption=false,font=footnotesize,labelfont=sf,textfont=sf]{subfig}
\else
  \usepackage[caption=false,font=footnotesize]{subfig}
\fi
\usepackage{url}
\usepackage{color}
\newtheorem{thm}{Theorem}[section] 
\newtheorem{lemma}[thm]{Lemma}
\newtheorem{theorem}[thm]{Theorem}
\newtheorem{corollary}[thm]{Corollary}

\newcommand{\cX}{{\cal{X}}}

\newcommand{\cF}{{\cal{F}}}
\newcommand{\cY}{{\cal{Y}}}
\newcommand{\cD}{{\cal{D}}}

\newcommand{\bx}{\mathbf{x}}

\newcommand{\reals}{\mathbb{R}}
\makeindex
\begin{document}
%
% paper title
% Titles are generally capitalized except for words such as a, an, and, as,
% at, but, by, for, in, nor, of, on, or, the, to and up, which are usually
% not capitalized unless they are the first or last word of the title.
% Linebreaks \\ can be used within to get better formatting as desired.
% Do not put math or special symbols in the title.
\title{Learn on Source, Refine on Target: \\ A Model Transfer Learning Framework with Random Forests}
%
%
% author names and IEEE memberships
% note positions of commas and nonbreaking spaces ( ~ ) LaTeX will not break
% a structure at a ~ so this keeps an author's name from being broken across
% two lines.
% use \thanks{} to gain access to the first footnote area
% a separate \thanks must be used for each paragraph as LaTeX2e's \thanks
% was not built to handle multiple paragraphs
%
%
%\IEEEcompsocitemizethanks is a special \thanks that produces the bulleted
% lists the Computer Society journals use for "first footnote" author
% affiliations. Use \IEEEcompsocthanksitem which works much like \item
% for each affiliation group. When not in compsoc mode,
% \IEEEcompsocitemizethanks becomes like \thanks and
% \IEEEcompsocthanksitem becomes a line break with idention. This
% facilitates dual compilation, although admittedly the differences in the
% desired content of \author between the different types of papers makes a
% one-size-fits-all approach a daunting prospect. For instance, compsoc 
% journal papers have the author affiliations above the "Manuscript
% received ..."  text while in non-compsoc journals this is reversed. Sigh.

%\author{Michael~Shell,~\IEEEmembership{Member,~IEEE,}
%        John~Doe,~\IEEEmembership{Fellow,~OSA,}
%        and~Jane~Doe,~\IEEEmembership{Life~Fellow,~IEEE}% <-this % stops a space
\author{Noam~Segev, Maayan Harel, Shie Mannor, Koby Crammer and Ran El-Yaniv% <-this % stops a space
\IEEEcompsocitemizethanks{
\IEEEcompsocthanksitem N. Segev and R. El-Yaniv are with the Department of Computer Sciences, Technion-Israel Institute of Technology, Israel\protect\\
E-mail: [nsegev, rani]@cs.technion.ac.il
%M. Shell is with the Department
%of Electrical and Computer Engineering, Georgia Institute of Technology, Atlanta,
%GA, 30332.\protect\\
% note need leading \protect in front of \\ to get a newline within \thanks as
% \\ is fragile and will error, could use \hfil\break instead.
%E-mail: see http://www.michaelshell.org/contact.html
\IEEEcompsocthanksitem M. Harel, S. Mannor and K. Crammer are with the Department of Electrical Engineering, Technion-Israel Institute of Technology, Israel\protect\\
E-mail: [maayanga,shie,koby]@ee.technion.ac.il}% <-this % stops an unwanted space
%\thanks{Manuscript received April 19, 2005; revised September 17, 2014.}
}

% note the % following the last \IEEEmembership and also \thanks - 
% these prevent an unwanted space from occurring between the last author name
% and the end of the author line. i.e., if you had this:
% 
% \author{....lastname \thanks{...} \thanks{...} }
%                     ^------------^------------^----Do not want these spaces!
%
% a space would be appended to the last name and could cause every name on that
% line to be shifted left slightly. This is one of those "LaTeX things". For
% instance, "\textbf{A} \textbf{B}" will typeset as "A B" not "AB". To get
% "AB" then you have to do: "\textbf{A}\textbf{B}"
% \thanks is no different in this regard, so shield the last } of each \thanks
% that ends a line with a % and do not let a space in before the next \thanks.
% Spaces after \IEEEmembership other than the last one are OK (and needed) as
% you are supposed to have spaces between the names. For what it is worth,
% this is a minor point as most people would not even notice if the said evil
% space somehow managed to creep in.

% The paper headers
\markboth{Learn on Source, Refine on Target: A Model Transfer Learning Framework with Random Forests}%
{Learn on Source, Refine on Target: A Model Transfer Learning Framework with Random Forests}
\IEEEtitleabstractindextext{%
\begin{abstract}
We propose novel \emph{model transfer-learning} methods that refine a decision forest model $M$ learned within a ``source'' domain using a training set sampled from a ``target'' domain, assumed to be a variation of the source. 
We present two random forest transfer algorithms. 
The first algorithm searches greedily for locally optimal modifications of each tree structure by trying to 
locally expand or reduce the tree around individual nodes. 
The second algorithm does not modify structure, but only the parameter (thresholds) associated with decision nodes. 
We also propose to combine both methods by considering an ensemble that contains the union of  the two forests. 
The proposed methods exhibit impressive experimental results over a range of problems.
\end{abstract}

% Note that keywords are not normally used for peerreview papers.
}

% make the title area
\maketitle

% To allow for easy dual compilation without having to reenter the
% abstract/keywords data, the \IEEEtitleabstractindextext text will
% not be used in maketitle, but will appear (i.e., to be "transported")
% here as \IEEEdisplaynontitleabstractindextext when the compsoc 
% or transmag modes are not selected <OR> if conference mode is selected 
% - because all conference papers position the abstract like regular
% papers do.
\IEEEdisplaynontitleabstractindextext
% \IEEEdisplaynontitleabstractindextext has no effect when using
% compsoc or transmag under a non-conference mode.

% For peer review papers, you can put extra information on the cover
% page as needed:
% \ifCLASSOPTIONpeerreview
% \begin{center} \bfseries EDICS Category: 3-BBND \end{center}
% \fi
%
% For peerreview papers, this IEEEtran command inserts a page break and
% creates the second title. It will be ignored for other modes.
\IEEEpeerreviewmaketitle

\IEEEraisesectionheading{\section{Introduction}\label{sec:introduction}}
% Computer Society journal (but not conference!) papers do something unusual
% with the very first section heading (almost always called "Introduction").
% They place it ABOVE the main text! IEEEtran.cls does not automatically do
% this for you, but you can achieve this effect with the provided
% \IEEEraisesectionheading{} command. Note the need to keep any \label that
% is to refer to the section immediately after \section in the above as
% \IEEEraisesectionheading puts \section within a raised box.

% The very first letter is a 2 line initial drop letter followed
% by the rest of the first word in caps (small caps for compsoc).
% 
% form to use if the first word consists of a single letter:
% \IEEEPARstart{A}{demo} file is ....
% 
% form to use if you need the single drop letter followed by
% normal text (unknown if ever used by IEEE):
% \IEEEPARstart{A}{}demo file is ....
% 
% Some journals put the first two words in caps:
% \IEEEPARstart{T}{his demo} file is ....
% 
% Here we have the typical use of a "T" for an initial drop letter
% and "HIS" in caps to complete the first word.

%\IEEEPARstart{I}{n} 
\IEEEPARstart{C}{onsider} a software company selling a trained predictive model $M$ to a community of consumers.
The generic classifier $M$ was constructed using a very large and expensive dataset $D$. 
While the generic classifier is very accurate over the ``source" $D$, each of the individual consumers needs to apply $M$
in a specific ``target" context $D'$ with its own idiosyncrasies and noise parameters. 
The manufacturer can neither share its dataset $D$ with its consumers
nor afford to retrain an individual model for each of them (based on both $D$ and $D'$).
What would be a good approach to adapt the model $M$ to each individual context using a relatively small training set?

% A mini discussion on model transfer and its challenges
In this paper we focus on the setting of \emph{model transfer} (MT) whereby
the adaptation of a given source model to a target domain relies on a relatively small training set from the target. 
In contrast to general transfer learning frameworks
(such as \emph{instance transfer}, see Sec.~\ref{sec:comparingToInstance}),
in model transfer no training examples are available from the source domain during adaption for whatever reason, 
e.g., storage capacity or data privacy.
This limitation makes model transfer a restrictive and more challenging type of transfer learning.

% Motivation - general
There are numerous practical scenarios where model transfer is essential, 
whereas the source/target data sharing required by standard transfer learning methods is impermissible. 
For example, Microsoft's Kinect performs human pose recognition using random forests \cite{shotton2013real}, 
which can be improved with user-specific training data to accommodate environmental changes (e.g., lighting and furniture) or mechanical ones. 
One of our experimental settings addresses a conceptually similar case.
In general, when the model manufacturer cannot send the data to the model consumer or the consumer cannot send the data to the manufacturer, be it due, for example, 
to memory limitations (at the consumer's box) or computational/communication constraints (at/to the manufacturer's site), 
model transfer is interesting as a potentially viable solution.
It is certainly conceivable that model transfer for machine learning will be extremely widespread in the near future.

% Motivation - ours
Our own motivation to consider model transfer arose in a collaborative project with a leading cyber fraud detection company dealing with online bank transactions.
The company created a powerful fraud detection model based on data collected from a number of banks (the ``source'' domain).
However, new clients (banks) operate in different contexts (the ``target'' domains), 
where the type of fraud committed might differ from the generic frauds, due to variations in transaction protocols, 
geo-demographics and other factors. 
Due to regulations and secrecy requirements, the company is forbidden to share its dataset with its clients, 
and many of the clients were forbidden to share their own datasets with the company. 

% The current existing approaches
While MT isn't new, no single set of assumptions exist that define the model transfer setting, 
and thus existing model transfer techniques vary in their approaches. 
However, model transfer techniques typically resort to regularizing the learning of the target domain using the model learned for the source domain.
This can be achieved, e.g., by using a biased regularizer \cite{liao2005logistic, kienzle2006personalized, yang2007cross, duan2009domain}, 
or aggregating multiple source-target predictors \cite{rettinger2006boosting, luo2008transfer, ruckert2008kernel}. 
A potential drawback of this regularization paradigm is its limited capacity to accommodate local changes between the source and target distributions, 
as these techniques typically focus on optimizing a global regularization.

% Why DTs are useful for model transfer and what motivated us
In contrast, the techniques we developed emphasize simple model transformations based on local 
(and greedy) changes.
We propose novel model transfer techniques that rely on decision trees (DTs).
As non-linear models, DTs can excel in learning non-linear decision rules,
and their hierarchical structure enables detection and accommodation of non-linear transformations from source to target.
Our methods are motivated by two frequently occurring transformations between the source and target domains:
(1) translations (shifts) of the distributions of individual features, 
and (2) transformations in which a set of features needs to be refined or made coarser to fit the target problem.
Each of the two DT techniques we propose is designed to capture one of these types.

% Why DT ensembles are necessary to avoid overfitting
In general, when dealing with DT learning, one has to carefully guard against overfitting. 
The two common techniques to regularize DTs are pruning and voting ensembles \cite{mehta1995mdl, zhang2012ensemble}. 
Pruning is a technique to reduce the size of an existing decision tree by replacing internal nodes in the tree with leaf nodes. 
By reducing the tree size, one reduces the complexity of the classifier and hopefully removes sections of the tree that were based on a few noisy samples.
However, we utilize voting ensembles, 
whereby multiple trees are generated and a forest is built by applying our DT induction algorithms. 

% Lets mix the ensembles
Each of these forests alone can be used for transfer learning, 
but we observed that the two methods tend to excel in different problems.
While a judicious use of these algorithms based on prior knowledge of the problem at hand may suffice, 
we propose to create a diverse ensemble consisting of the union of the models generated by both methods.
The resulting  algorithm is capable of modeling complex, real world, source-to-target transformations,
while performing better than, or almost as well as, its underlying constituents in most cases.
This union algorithm is relatively easy to implement, 
requires modest hyper-parameter tuning, can effectively exploit intensive computational resources to handle large-scale problems, and improves state-of-the-art performance on a range of problems.

% The paper structure
After introducing our learning setup in Sec.~\ref{sec:notations}, 
we present the two new algorithms in Sec.~\ref{sec:algs} 
and provide insights into the algorithms' strengths and weaknesses using synthetic examples. 
In Sec.~\ref{sec:experiments}, we present extensive experiments over a number of datasets.
The results demonstrate the effectiveness of our method, 
which often outperforms several state-of-the-art transfer learning methods.
Related work is discussed in Sec.~\ref{sec:related},
followed by discussion about the advantage of our methods can be found in Sec.~\ref{sec:explain-mix}. 

\section{Preliminary Definitions}
\label{sec:notations}

% Domains and transfer learning
A \emph{domain} $\cD = \left( \cX, \cY, P \right)$ consists of an $r$-dimensional feature space $\cX$,
a label space $\cY$, and a probability distribution $P \left(\bx, y \right)$, 
where $\bx \in \cX$ is the feature vector, and $y \in \cY$.
In \emph{supervised model transfer learning}, we are given two domains: 
a source domain, $\cD_S = \left( \cX_S, \cY_S, P_S \right)$, 
and a target domain, $\cD_T = \left( \cX_T, \cY_T, P_T \right)$.
Given a loss function $\ell : \cY \times \cY \to \reals^+$, 
a source prediction function $f : \cX_S \to \cY_S$ 
and (typically limited) target training set $S^T$ drawn i.i.d.~from $\cD_T$,
our objective is to learn a function $f \in \cF_{T} : \cX_T \to \cY_T$ 
with low risk $R(f) = E_{P_T(\bx,y)}\{ \ell( f(\bx), y) \}$ on the target domain. 
%. \footnote{Using learning models that memorize the source training set (e.g., nearest 
%neighbors) clearly make the restriction in model transfer vacuous.}

% Some of our assumptions
Different transfer learning models have different restrictions on the relationship between the source and the target domains. 
Our work focuses on \emph{inductive transfer learning}, 
a setting in which one assumes that both source and target tasks share the same features and label spaces,
i.e., $\cX_S, \cX_T \subseteq \cX$ and $\cY = \cY_S = \cY_T$. 
Clearly, the marginal distribution of the features may differ between the domains. 
This setting is quite common, both in research literature and in real world applications of transfer learning, 
but is only one of a few existing approaches \cite{pan2010survey}.
The presented framework is suitable for both binary and multi-class classification tasks where $\ell$ is the zero-one loss function. 

\subsection{Random Forests Models:}
% Notations for the decision trees
Our algorithms are based on standard Random Forests (RFs) \cite{breiman2001random}.
We use the following notations for a tree in the forest:
Each tree node $v$ has an out-degree $d(v)$ and its children are denoted $v_1, \ldots, v_{d(v)}$. 
A leaf node $v$ is associated with a single decision value in $\cY$, denoted $y(v)$. 
An internal (non-leaf) node $v$ is associated with a single feature $\phi(v)$, 
and for a numeric feature it is also associated with a numeric threshold $\tau(v)$. 
Classification of a sample is done based on the leaf that sample reaches, 
i.e., the leaf at the end of the path in the tree the sample will follow, 
and for each node $u$ along this path we say that $\bx$ ``arrives'' at $u$.

\section{Algorithms}
\label{sec:algs}
We now describe our two algorithms, SER and STRUT, for refitting trees to the target domain. 

\subsection{Structure Expansion/Reduction}
\label{sec:SER}

\begin{algorithm}[tb]
   \caption{Structure Expansion Reduction (SER) \label{algo:SER}}
\begin{algorithmic}
   \STATE {\bfseries Input:} Node $v$, labeled samples $S_v$
   \STATE {\bfseries Output:} Node $v$

   \STATE \% \emph{Expand leaves:}
   \IF{$d(v) = 0$} 
   \STATE $v \gets$ Build Tree $\left( S_v \right)$
   \RETURN $v$
   \ENDIF

   \STATE \% \emph{Recurse over child nodes:}
   \FOR{$v_i \in \left\{ v_1, \ldots, v_n \right\}$}
   \STATE Structure Expansion Reduction $\left( v_i, S_{v_i} \right)$
   \ENDFOR

   \STATE \% \emph{Reduce current node:}
   \IF{leafError$\left(v, \left( S_v \right) \right) <$ subtreeError$\left(v, \left( S_v \right) \right)$} 
    \FOR{$i \in d(v)$}
     \STATE deleteNode($v_i$)
    \ENDFOR
    \STATE $d(v) \gets 0$; $y(v) \gets \underset{y}{\arg\max} \left| \left\{ \left( \cdot, y \right) \in S_v \right\} \right|$
   \ENDIF

   \RETURN $v$
\end{algorithmic}
\end{algorithm}

% What is SER?
The \emph{structure expansion/reduction} (SER) algorithm pairs two local transformations of a decision tree structure: expansion and reduction. 
In the \emph{expansion} transformation, we specialize rules induced over the source data to the target data.
In \emph{reduction} we perform the opposite operation, i.e., generalize rules induced over the source data. 

% How SER works
Initially, a random forest is induced using the source data $S^S$.
The SER algorithm begins by calculating for each node $v$ the set $S^T_v$ of all labeled points
in the target data $S^T$ that reaches $v$.
Then, each leaf $v$ is expanded to a full tree with respect to the sample set $S^T_v$.
Finally, for each internal node $v$, the algorithm --- working bottom-up on the tree ---
attempts to perform a structure reduction as follows. 
Two types of errors are defined for node $v$ with respect to $S^T_v$.
The first, the \emph{subtree error}, is the empirical error of the subtree whose root is $v$.
The second, the \emph{leaf error},
is defined to be the empirical error on $v$ if it were to be pruned into a leaf.
If the leaf error is smaller than the subtree error, the subtree is pruned into a leaf node.
The decision value at each leaf of the modified DT is obtained using the target 
(empirical) distribution.

A pseudo-code of this  algorithm is presented in Algorithm~\ref{algo:SER}.
Note that the recursive calls are equivalent to depth-first traversal, with expansion performed whenever a leaf is reached, 
followed by a reduction step upon the completion of all recursive calls to an internal node's children.

\subsubsection{Visual Illustration of the SER Algorithm}

\begin{figure}[t]
 \centering
 \subfloat[width=1\linewidth][\centering single box prior to splitting \label{fig:SER-example-a}]{
  \includegraphics[width=0.3\linewidth]{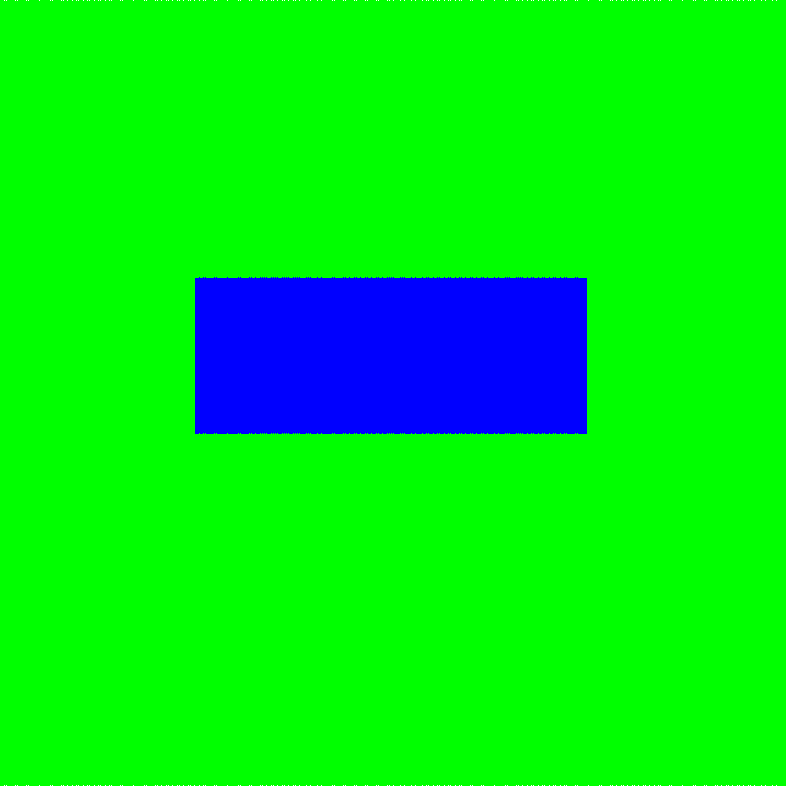}
 }
 \hfill
 \subfloat[width=1\linewidth][\centering boxes after splitting \label{fig:SER-example-b}]{
  \includegraphics[width=0.3\linewidth]{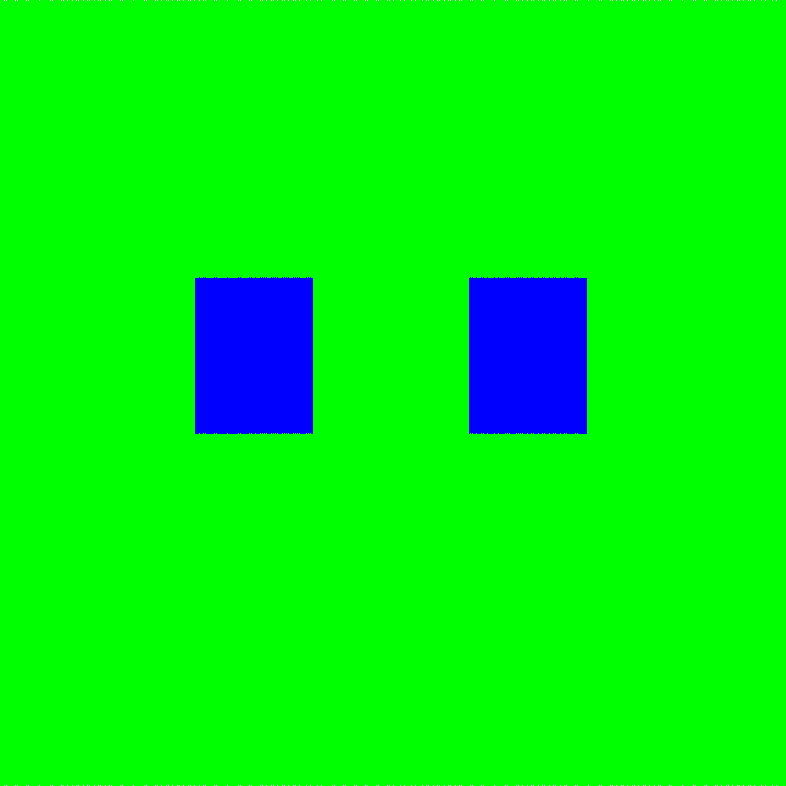}
 }
 \hfill
 \par \medskip
 \subfloat[width=1\linewidth][\centering decision trees matching each domain\label{fig:SER-example-c}]{
  \includegraphics[width=0.9\linewidth]{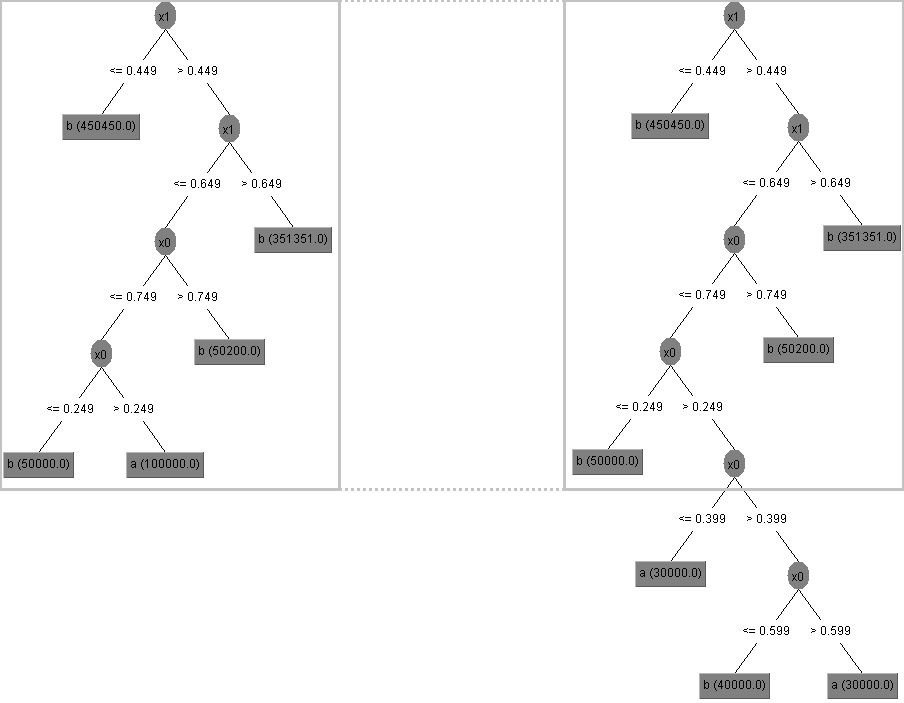}
 }
 \caption{Simple box example and resulting decision trees\label{fig:SER-example}}
\end{figure}

The SER algorithm applies two local transformations on a given decision tree.
To gain some intuition about its operation, we exemplify these transformations.
In Figures~(\ref{fig:SER-example-a},\ref{fig:SER-example-b}) we depict two simple domains (classification problems).
The DTs learned for each domain are illustrated in Figure~(\ref{fig:SER-example-c}).
A standard DT induced over (\ref{fig:SER-example-a}) resulting in the LHS tree in Figure~(\ref{fig:SER-example-c}) can be easily transferred to 
domain (\ref{fig:SER-example-b}) using the expansion operation (applied by SER), resulting in the 
RHS tree of Figure~(\ref{fig:SER-example-c}) .
Similarly, the tree induced for (\ref{fig:SER-example-b}) can be transferred to 
(\ref{fig:SER-example-a}) using the reverse, reduction operation. 

It is obvious that a single classifier can describe two identical domains. 
Therefore, as one domains drifts, the changes can be captured via small modifications to the tree structure. 
The given example demonstrates this simple and intuitive observation, 
showing the high similarity between tree models.
As the concepts drift further apart, iterative SER transformations can capture the new domain 
while maintaining a high degree of correlation between the unchanged similar sections of the domains.

\subsubsection{Logical Regularization in SER}

The SER algorithm is especially designed to first consider expansions and then reductions.
In this section we explain the rationale for this design and argue that it serves as a kind of
regularization that keeps the resulting target model closer to the source model 
than it would if reduction were to precede expansion.

It is well known that a decision tree is equivalent to a disjunctive normal form (DNF) formula 
where a single rule $\tau$, which constitutes a root-to-leaf path, 
is equivalent to a conjunction of literals \cite{breiman1984classification}.
Let $u=u_0, \ldots, u_m$ be a root-to-leaf path in a tree prior to running the SER algorithm, 
with rule $\tau_S$ corresponding to this path.
Let $u'=u'_0, \ldots, u'_n$ be a root-to-leaf path in a tree after running the SER algorithm, 
with rule $\tau_T$ corresponding to this path.
If the path $u'$ was generated from the path $u$ by a SER expansion step, then 
$u_i = u'_i$ for $0 \leq i \leq m$, and we say that rule $\tau_T$ expands rule $\tau_S$.
Similarly, if the path $u'$ was generated from the path $u$ by a SER reduction step, then 
$u_i = u'_i$ for $0 \leq i \leq n$, and we say that rule $\tau_T$ reduces rule $\tau_S$.

Following these definitions, we make two observations on the relations between $\tau_T$ and $\tau_S$. 
\begin{lemma}
\label{lem:SER-exp}
If rule $\tau_T$ expands rule $\tau_S$, then $\tau_T$ satisfies $\tau_S$ 
(i.e., if $\bx \in \cX$ satisfies $\tau_T$ then it also satisfies $\tau_S$). 
\end{lemma}
\begin{proof}%[lem:natural:blowups-preserve-distance-on-average]
\label{proof:SER-exp}
Let $u=u_0, \ldots, u_n$ be the path corresponding to $\tau_T$ and 
$u'=u'_0, \ldots, u'_m$ the path corresponding to $\tau_S$.
Rule $\tau_T$ is comprised of $n$~literals and rule $\tau_S$ consists of $m$~literals;
each literal corresponds to a single node along a path.
As $\tau_T$ is a conjunction of literals, 
$\bx \in \cX$ satisfies $\tau_T$ if and only if $\bx$ satisfies all of the $n$ terms in $\tau_T$. 
As $u_i = u'_i$ for $0 \leq i \leq m$, the $m$ terms of rule $\tau_S$ are among the $n$ terms in rule $\tau_T$. 
Thus, if $\bx$ satisfies all of the $n$ terms of $\tau_T$, it also satisfies the $m$ terms which appear in both rules, 
and as $\bx$ satisfies all of the $m$ terms in $\tau_S$, $\bx$ satisfies $\tau_S$.
\end{proof}

\begin{lemma}
\label{lem:SER-red}
If rule $\tau_T$ reduces rule $\tau_S$, then $\tau_S$ satisfies $\tau_T$.
\end{lemma}
\begin{proof}%[lem:natural:blowups-preserve-distance-on-average]
Similar to Proof~\ref{proof:SER-exp}, where the $n$ terms in rule $\tau_T$ are among the $m$ terms in rule $\tau_S$.
\end{proof}

Following Lemma~\ref{lem:SER-exp} and Lemma~\ref{lem:SER-red}, 
the operation order of expansion followed by reduction has an interesting property:
\begin{corollary}
\label{cor:SER-log}
For any rule $\tau_T$ in the transformed tree, there exists a rule $\tau_S$ in the original tree, 
such that either $\tau_T$ satisfies $\tau_S$ or $\tau_S$ satisfies $\tau_T$.
\end{corollary}

The property presented in Corollary~\ref{cor:SER-log} is desirable in our context where we intend to perform local refinements and/or generalizations.
In contrast, this property is violated when applying first reduction and then expansion, 
in which case the resulting model can drift further away from the source model.

\subsection{Structure Transfer}

\begin{algorithm}[t]
 \caption{Structure Transfer (STRUT) \label{algo:STRUT}}
 \begin{algorithmic}
  \STATE {\bfseries Input:} Node $v$, labeled samples $S$
  \STATE {\bfseries Output:} Node $v$

  \STATE \% \emph{Prune unreachable subtree:}
  \IF{($\left| S \right| = 0$)} 
   \FOR{( $i \in d(v)$ )}
    \STATE deleteNode($v_i$)
   \ENDFOR
   \STATE $d(v) \gets 0$
   \RETURN $v$
  \ENDIF

  \STATE \% \emph{Update leaf distribution:}
  \IF {($d(v) = 0$)}
   \STATE{$y(v) \gets \underset{y}{\arg\max} \left| \left\{ \left( \cdot, y \right) \in S \right\} \right|$}
   \STATE{return $v$}
  \ENDIF

  \STATE \% \emph{Refit thresholds for numeric features:}
  \IF{($\phi(v)$ is numeric)}
   \STATE $\tau_1 \leftarrow$ Threshold Selection$\left( S, \phi(v), Q_L(v), Q_R(v) \right)$
   \STATE DG$_1$ = DG $\left( S, \phi(v), \tau_1, Q_L(v), Q_R(v) \right)$
   \STATE $\tau_2 \leftarrow$ Threshold Selection$\left( S, \phi(v), Q_R(v), Q_L(v) \right)$
   \STATE DG$_2$ = DG $\left( S, \phi(v), \tau_2, Q_R(v), Q_L(v) \right)$
   \IF{(DG$_1 \geq$ DG$_2$)}
    \STATE $\tau(v) \leftarrow \tau_1$
   \ELSE
    \STATE $\tau(v) \leftarrow \tau_2$
    \STATE swap$\left( v_1, v_2 \right)$
   \ENDIF
  \ENDIF

  \STATE \% \emph{Run STRUT on sons:}
  \FOR{( $i \in d(v)$ )}
   \STATE STRUT $\left( v_i, S_{v_i} \right)$
  \ENDFOR
  \RETURN $v$
 \end{algorithmic}
\end{algorithm}

The \emph{structure transfer} (STRUT) algorithm is motivated by the observation that 
decision trees for similar problems should exhibit structural similarity. 
Consider, for example, the similar problems of detecting online fraud in two big banks in two different
geo-demographic environments (say one is in the USA and the other is in India). 
Both problems can be modeled such that they share many of the features and their dependencies 
(e.g., the feature `typical customer transaction size' should appear in both models). 
However, the scale of such features and their associated decision thresholds  are likely to differ between problems. 

The STRUT algorithm adapts a DT trained on the source samples to the target samples 
by discarding all numeric threshold values in the tree and working top-down, 
selecting a new threshold $\tau(v)$ for a node $v$ with a numeric feature using $S^T_v$, 
the subset of target examples that reach $v$. 
If the algorithm encounters a node $v$ for which $S^T_v$ is empty, 
$v$ is considered unreachable in the target domain and is pruned.
The final decision value at each leaf is computed on the target training data.

A pseudo-code of the STRUT algorithm is presented in Algorithm~\ref{algo:STRUT}. 
Threshold selection, as computed by the ThresholdSelection procedure, is posed as an optimization problem described bellow.

Any feature $\phi$ and threshold $\tau$ split any set of (labeled) examples, $S$, 
into two subsets, denoted $S_{L, \tau }$ and $S_{R, \tau }$. 
The label distributions over these subsets are denoted $Q_L$ and $Q_R$, respectively.
With respect to each decision node (with a corresponding feature $\phi$), 
STRUT's goal is to optimize a decision threshold with respect to the target training data. 
As an unconstrained optimization is not advisable in cases where the target training set is small, 
we require that the newly optimized threshold for decision node $v$ result 
in label distributions $Q'_L$ and $Q'_R$ that are similar to $Q_L$ and $Q_R$, 
the original distributions obtained when training $v$. 
To this end, we define the \emph{divergence gain} (DG) measure, presented in Equation~(\ref{eq:DG}),
that quantifies the similarity of the resulting distributions obtained for $v$, with respect to training set $S^T_v$, 
to the original distributions, $Q_L$ and $Q_R$. 
DG relies on the (symmetric) Jensen-Shannon divergence given in Equation~(\ref{eq:JSD}), 
where $D_{KL}(\cdot)$ is the familiar Kullback-Leibler divergence and $M$ is the mean distribution, 
$M = \frac{1}{2} \left( P + Q \right)$ \cite{lin1991divergence}.
The choice of the Jensen-Shannon divergence is justified by its frequent use as an effective statistic for the two-sample problem.
\begin{equation}
\label{eq:DG}
\begin{split}
 DG &\left( S^T_v, \phi(v), \tau(v), Q_L, Q_R \right) = \\ 
1 - &\frac{\left| S_{L, \tau } \right|}{\left| S^T_v\right| } JSD(Q'_L, Q_L) - 
\frac{\left| S_{R, \tau } \right|}{\left| S^T_v\right| } JSD(Q'_R, Q_R).
\end{split}
\end{equation}
\begin{equation}
\label{eq:JSD}
2JSD \left( P, Q \right) = 
D_{KL}\left(P || M \right) + D_{KL}\left(Q || M \right).
\end{equation}

To perform threshold selection for feature $\phi$, STRUT uses DG to quantify distributional similarity 
and the information gain (IG) criterion to measure a threshold's informative value \cite{hunt1966experiments}. 
For feature $\phi$, STRUT looks for a threshold yielding high similarity between the induced distributions and the original distributions calculated during the tree induction stage. 
This similarity is restricted to ``informative" thresholds where, for any sufficiently small $\epsilon > 0$, 
the IG of threshold $x$ is larger than the IG of any other $x'$ in the $\epsilon$-neighborhood of $x$, i.e., 
thresholds that are local maximums of IG. 
Thus, STRUT's threshold selection can be formulated as an optimization problem~(\ref{eq:opt-problem}).
\begin{equation}
\label{eq:opt-problem}
\begin{array}{ll}
\underset{x}{\text{Maximize}} & DG\left(S_v^T, \phi, x, Q_L, Q_R \right) \\
\text{s.t.} & x \in \reals \\
 & \forall x' \in (x-\epsilon,x+\epsilon): \\
 & \qquad IG\left(S_{v}^{T}, \phi, x\right) \geq IG\left(S_{v}^{T}, \phi, x' \right).
\end{array}
\end{equation}

Problem (\ref{eq:opt-problem}) is not convex and we solve it using a line search on a limited number of possible thresholds. 
We note that the space/time complexities incurred by this optimization 
are very similar to the space/time complexities incurred when maximizing the IG value during standard tree induction.
Note also that in order to calculate the DG value we require node $v$ to retain the distributions 
$Q_L$ and $Q_R$ that were computed during construction.

It turns out that in some transfer learning problems features not only change threshold values but, 
as the concept drifts, they may also change their meaning. 
For example, in a fraud detection problem the ``average transaction size" feature occasionally changes meaning as attackers change strategies to fool the detector. 
The STRUT algorithm easily accommodates such changes between the source and the target by solving the optimization problem a second time but with the original distributions, $Q_L$ and $Q_R$, reversed.
If this switch improves the DG value, STRUT swaps the nodes' subtrees in conjunction with optimizing the threshold.

\subsubsection{On the Use of IG and DG}

\begin{figure}[t]
 \centering
 \subfloat[width=1\linewidth][\centering IG value for thresholds\label{fig:STRUT-example-a}]{
  \fbox{\includegraphics[width=0.7\linewidth]{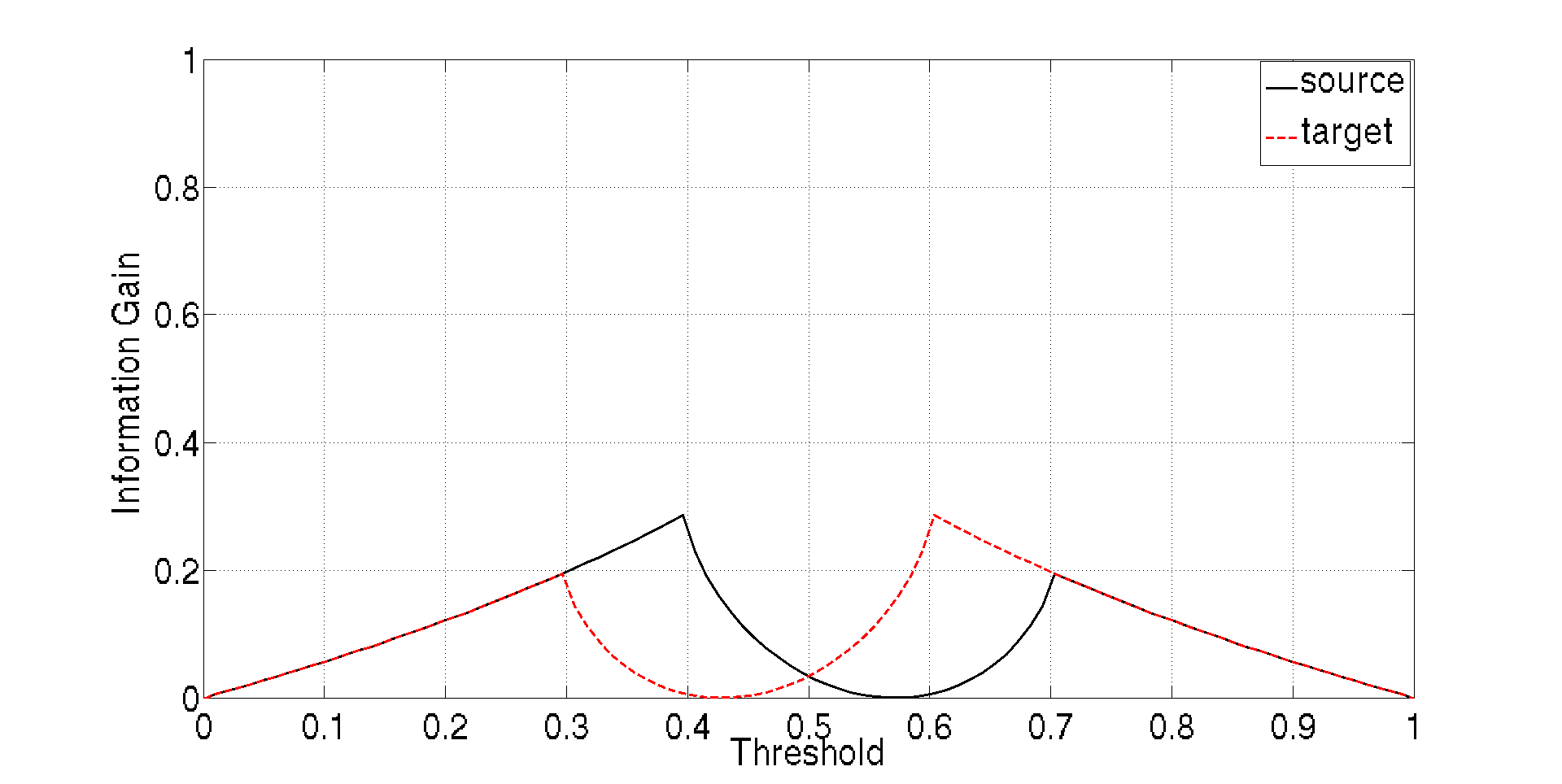}}
 }

 \subfloat[width=1\linewidth][\centering decision tree for source domain\label{fig:STRUT-example-b}]{
  \includegraphics[width=0.5\linewidth]{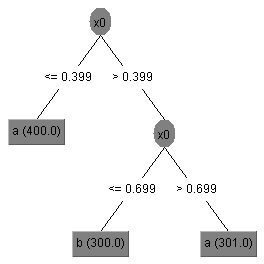}
 }
 \caption{Simple concept shift example \label{fig:STRUT-example}}
\end{figure}

As discussed above, the STRUT algorithm relies on both the DG and IG measures to optimize the adapted thresholds. 
Here we explain the motivation for using this combination of measures.
The IG is effective in quantifying the ``informativeness" of a threshold. 
However, IG is oblivious to dependencies enforced by the given structure.
In contrast, DG is a global regularization measure that does not account for local gains. 
The following examples show that each of these measures on its own fails to select an appropriate threshold.

Consider first the application of IG.
Define two simple domains where the feature space, $\cX$, is the range $\left[0 , 1\right]$,
and our label space is  $\cY = \left\{\pm 1\right\}$.
The source and target distributions, $P_s$ and $P_t$, are taken to be
$$
P_s\left(x\right) = \left\{ \begin{array}{rc}
1 & 0.4<x<0.7\\
-1 & otherwise;
\end{array}\right.
$$
$$
P_t\left(x\right) = \left\{ \begin{array}{rc}
1 & 0.3<x<0.6\\
-1 & otherwise.
\end{array}\right.
$$

The induced tree for the source domain is given in Figure~\ref{fig:STRUT-example-b} 
along with a plot of the IG values for different thresholds in Figure~\ref{fig:STRUT-example-a}.
Using a restricted variant of the STRUT algorithm on this problem, applied only with the IG measure, will result in 
a decision stump with a large error rate of $\sim 30\%$. 
The reason is that the root threshold is set by the algorithm to 0.6 and the left tree, 
which is a leaf, will just update the returned decision, while the right child will be pruned, 
because the samples that arrive at this node will all belong to a single label.

Next we show a simple concept shift problem where DG over-regularizes unless it is mitigated by local considerations.
We keep the same feature space as in the previous example (the range $\left[0 , 1\right]$) 
as well as the same label space ($\cY = \left\{\pm 1\right\}$).
However, now the source and target distributions, $P_s$ and $P_t$, are 
$$
P_s\left(x\right) = \left\{ \begin{array}{rc}
 1 & 0 \leq x < 0.5\\
-1 & 0.5 \leq x < 0.75\\
 1 & otherwise ;
\end{array}\right.
$$
$$
P_t\left(x\right) = \left\{ \begin{array}{rc}
 1 & 0 \leq x < 0.6\\
-1 & 0.6 \leq x < 0.85\\
 1 & otherwise .
\end{array}\right.
$$

\begin{figure}[t]
 \centering
 \includegraphics[width=.5\linewidth]{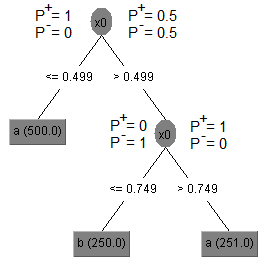}
  \caption{Induced decision tree with distributions \label{fig:STRUT-example2}}
\end{figure}

The induced tree for the source domain, as well as the induced distributions in each node $v$, 
are given in Figure~\ref{fig:STRUT-example2}.
Using a variant of the STRUT algorithm now restricted to apply only the DG measure will result in a tree whose error rate is $\sim 10\%$. 
The reason is that the root threshold is set by the algorithm to 0.5. 
The left tree is a leaf, which will result in updating the returned decision of the leaf (i.e., no change actually occurs).
However, for the right child, which is a stump, we are faced with a problem consisting of three distinct regions:
$$
\begin{array}{rc}
 1 & 0.5 \leq x < 0.6\\
-1 & 0.6 \leq x < 0.85\\
 1 & otherwise
\end{array}.
$$
If STRUT were to use the DG measure on its own, it would choose the threshold with the maximum DG value, which is $0.85$, 
and as the node's children, which are all leaves, it would simply update the returned decision (i.e., still no change occurs).
In this case, it is easy to see that the new tree misclassifies the range $0.5 \leq x < 0.6$.

\begin{table*}[t]
 \centering
 \caption{Test results of transfer forests on synthetic challenges}
 \subfloat[Two simple source-target transformations]{
  \label{tab:examples-synth-figs}
  \centering
  \begin{tabular}{r|c|c|ccc}
    & \multicolumn{1}{c}{$\cD_S$}  & \multicolumn{1}{c}{$\cD_T$}  & \multicolumn{1}{|c}{\bf STRUT}  & \multicolumn{1}{c}{\bf SER}  & \multicolumn{1}{c}{\bf MIX} \\
   \hline
   moving source & 
   \fbox{\includegraphics[width=0.06\textwidth]{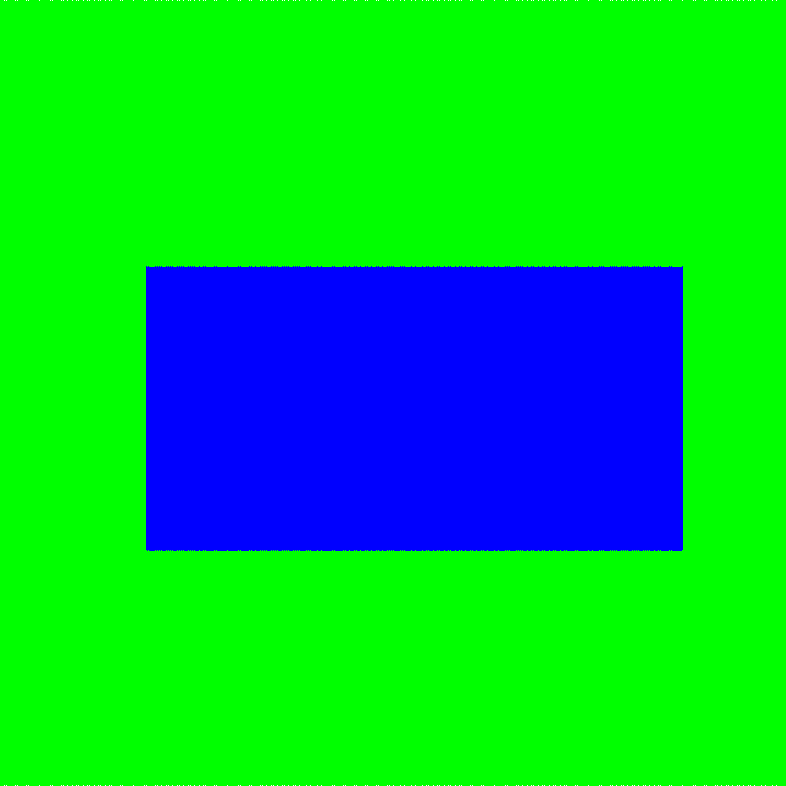}} & 
   \fbox{\includegraphics[width=0.06\textwidth]{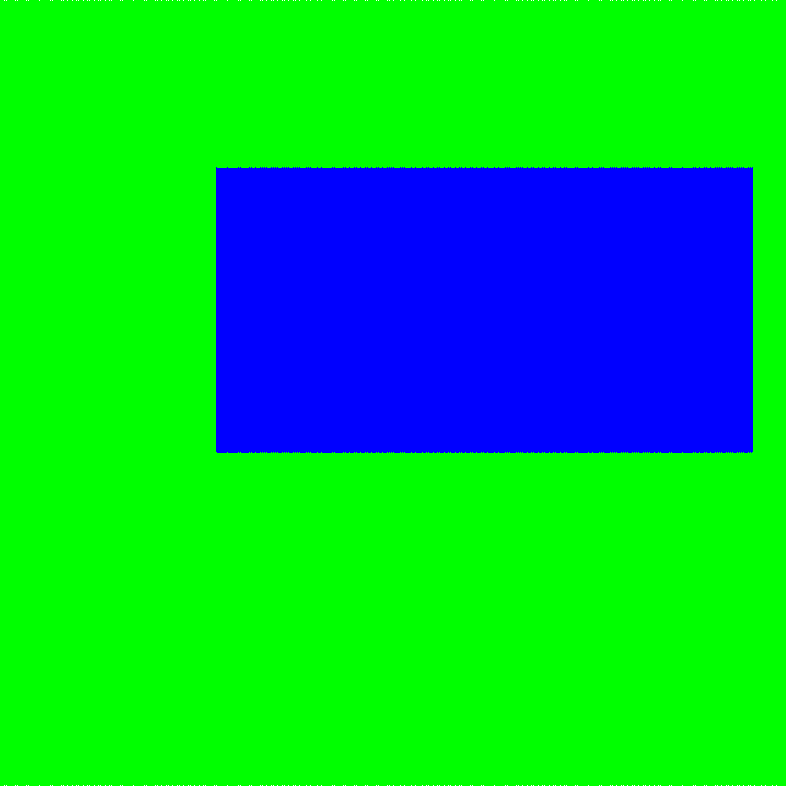}} & 
   \fbox{\includegraphics[width=0.06\textwidth]{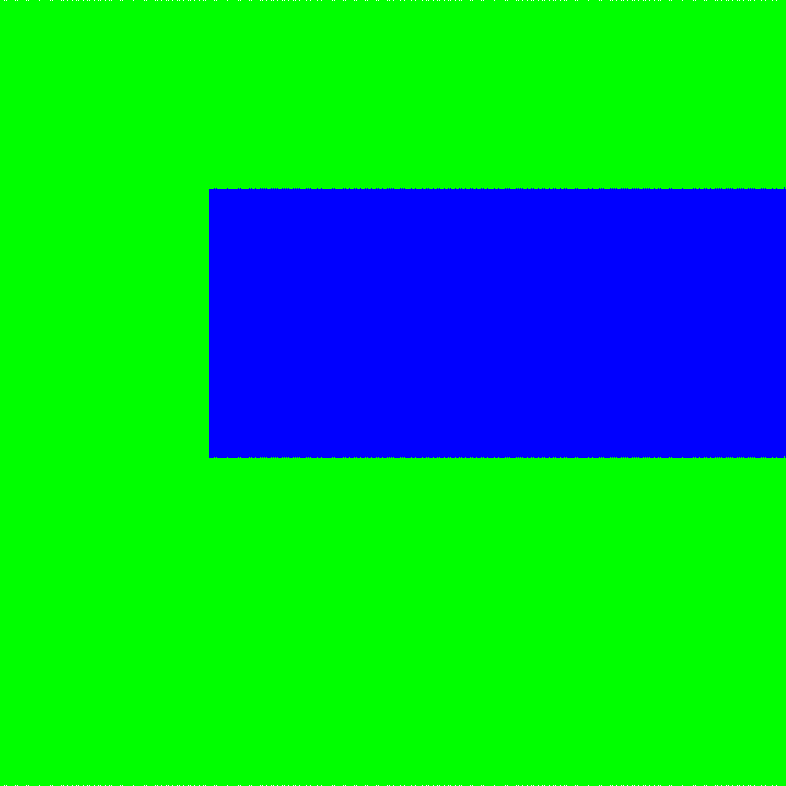}} & 
   \fbox{\includegraphics[width=0.06\textwidth]{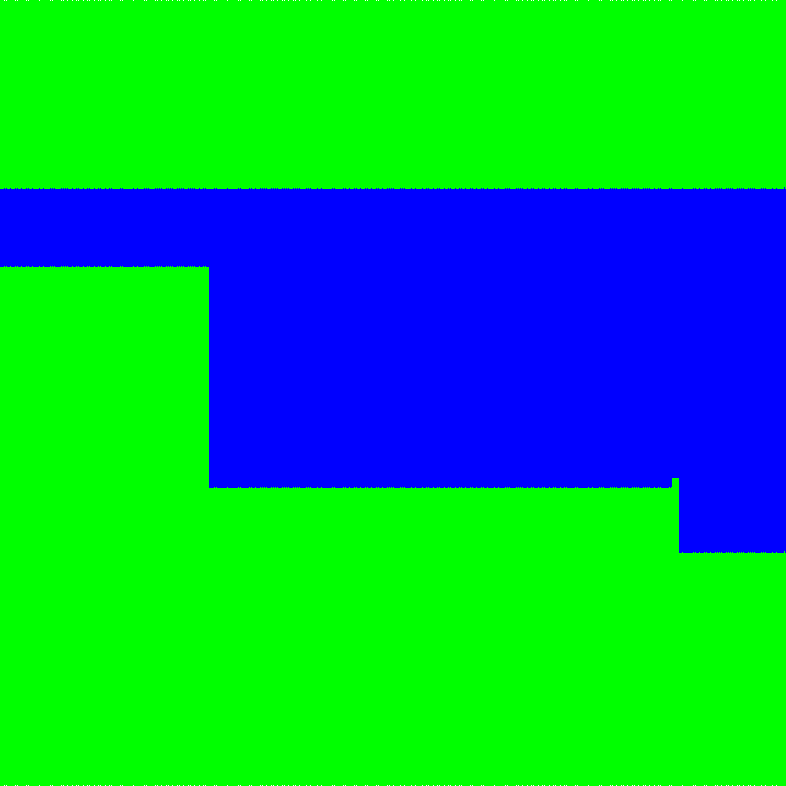}} & 
   \fbox{\includegraphics[width=0.06\textwidth]{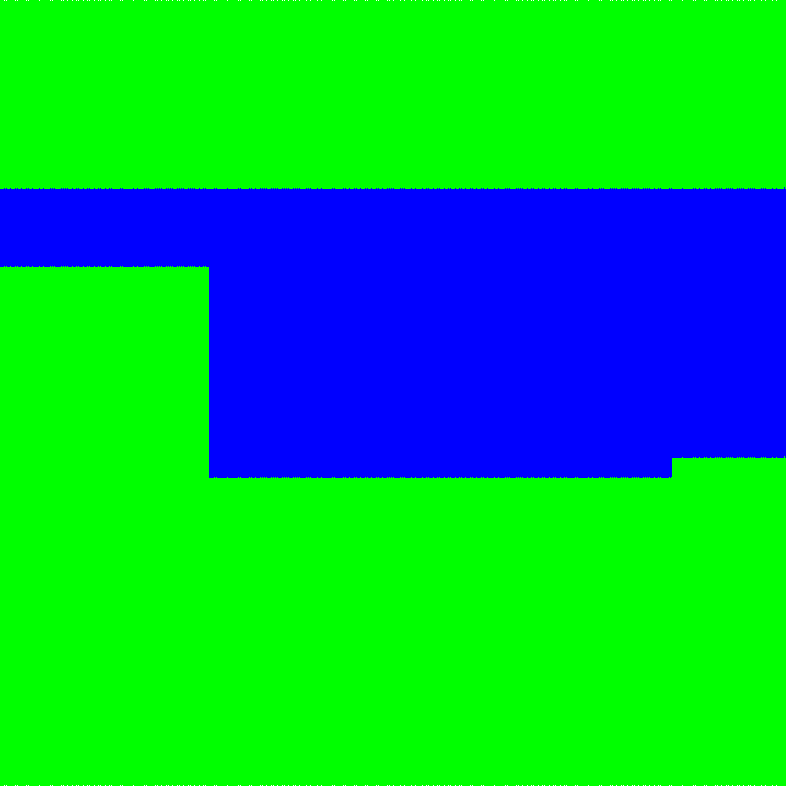}} \\
   \hline 
   mixed boxes & 
   \fbox{\includegraphics[width=0.06\textwidth]{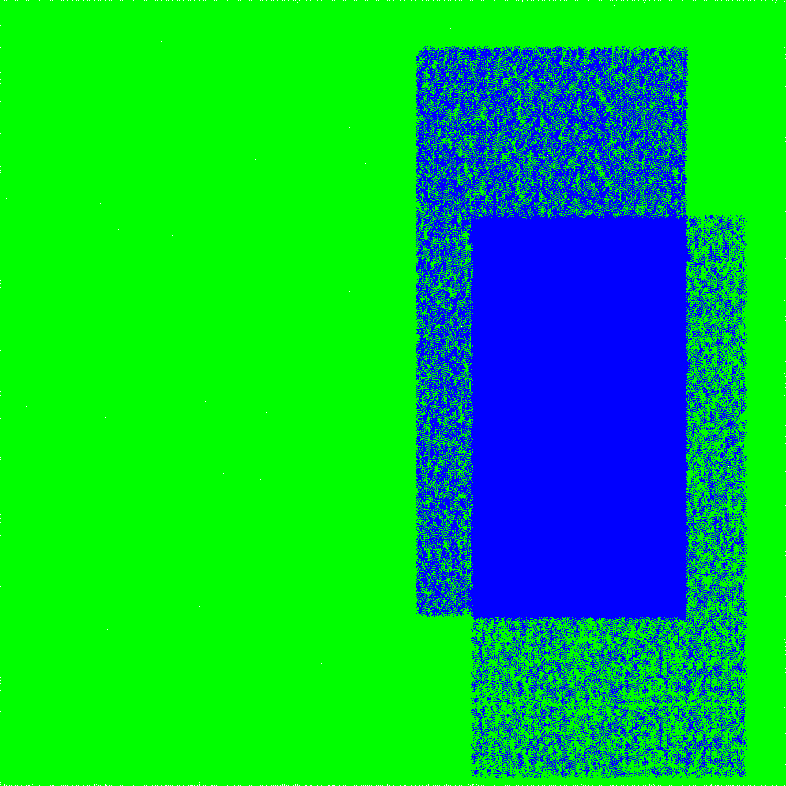}} & 
   \fbox{\includegraphics[width=0.06\textwidth]{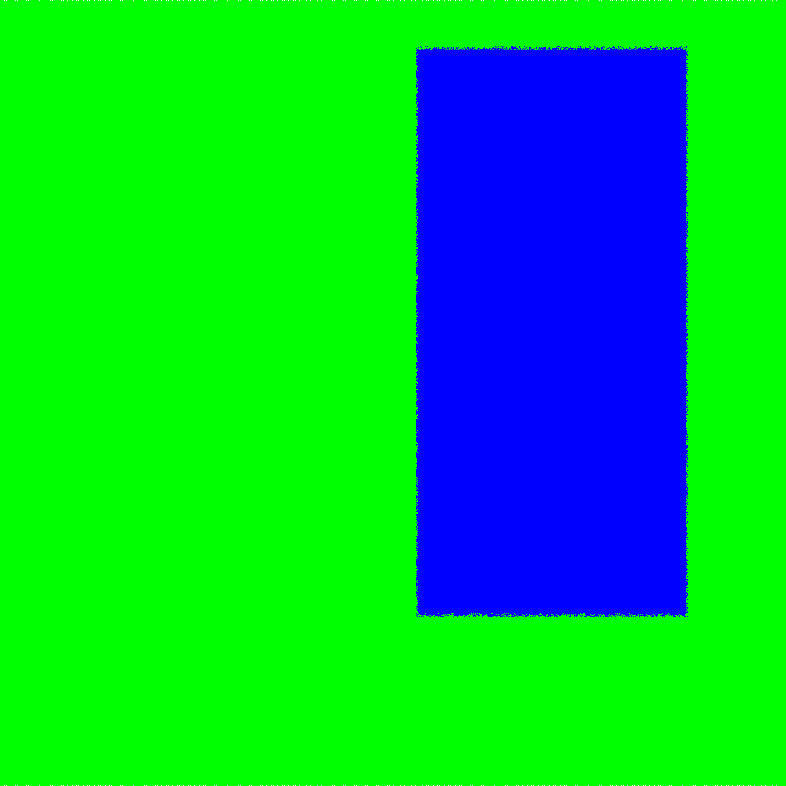}} & 
   \fbox{\includegraphics[width=0.06\textwidth]{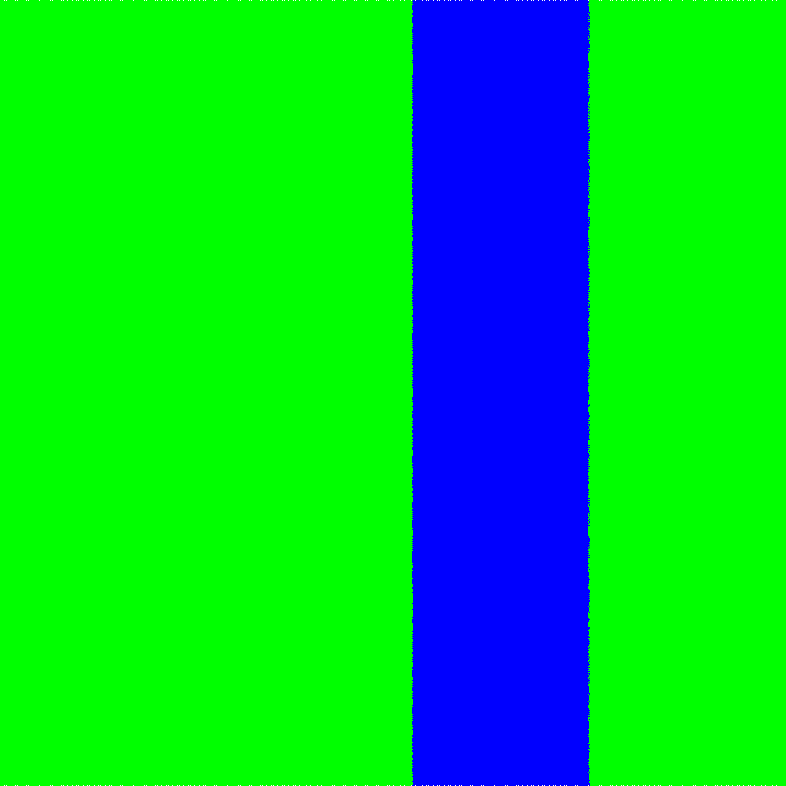}} & 
   \fbox{\includegraphics[width=0.06\textwidth]{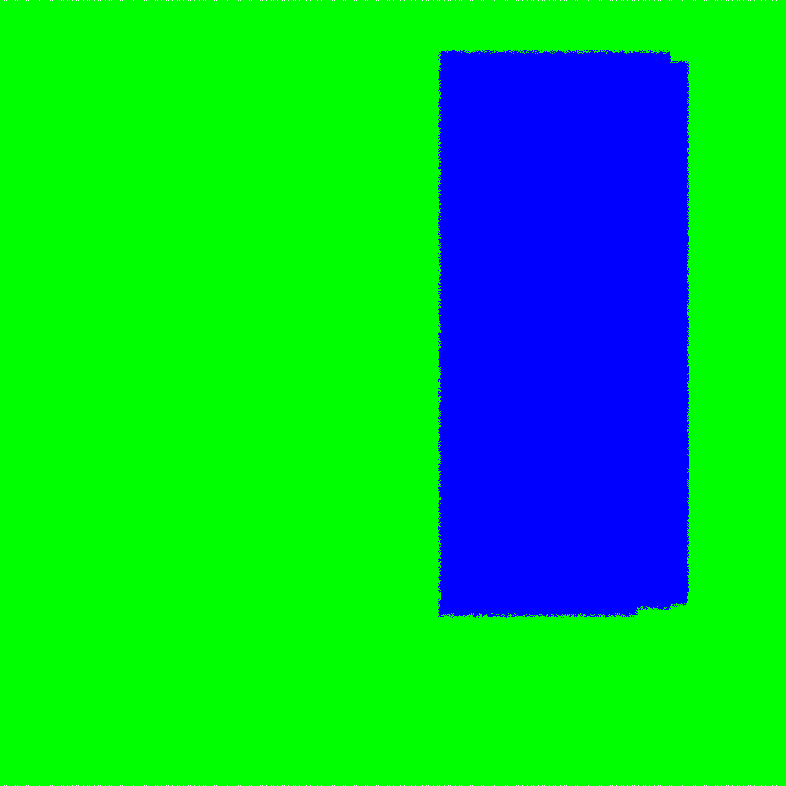}} & 
   \fbox{\includegraphics[width=0.06\textwidth]{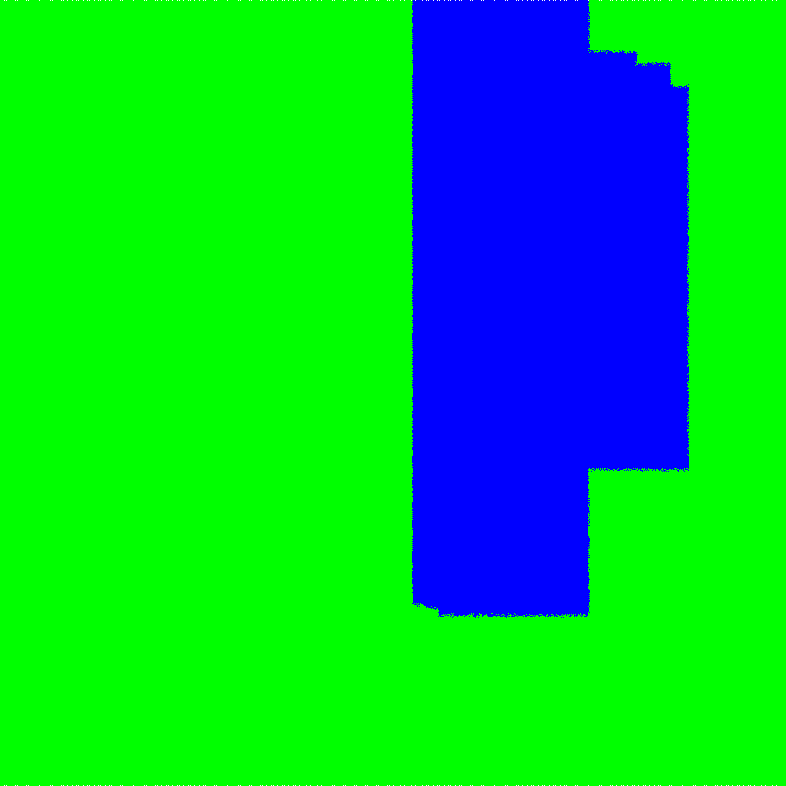}} \\
  \end{tabular}
 }
 \quad{}
 \subfloat[Test error rates - boldface marks lowest error]{
  \centering
  \begin{tabular}{l|ccc}
    & \multicolumn{1}{c}{\bf STRUT}  & \multicolumn{1}{c}{\bf SER}  & \multicolumn{1}{c}{\bf MIX} \\
  \hline 
  moving source & $\mathbf{6.1}$ & $12.8$ & $6.4$ \\
  \hline 
  mixed boxes & 7.7 & $6.6$ & $\mathbf{5.5}$ \\
  \end{tabular}
  \label{tab:experiments-synth-err}
 }
\end{table*}

It is not hard to see that in both the above negative examples (for using IG or DG on their own), 
the transformed trees can achieve $100\%$ accuracy if both the IG and DG measures are used in conjunction, 
as prescribed by the (unrestricted) STRUT algorithm.

\subsection{A MIX Approach}

Our proposed solution is to generate two forests using both SER and STRUT and then define MIX as a voting ensemble 
whose underlying model is the union of all the trees in these forests.
Thus, MIX is a simple majority vote applied over all decision trees transferred by either SER or STRUT.
As can be seen below, the resulting MIX ensemble often outperforms both of its constituents. 
An intuitive explanation for its excellent performance appears in Section~\ref{sec:explain-mix}.

\subsection{Numerical Examples and Intuition}
\label{sec:examples}

To gain intuition about the relative strengths and weaknesses of the SER and STRUT algorithms, 
as well as the  MIX solution, we consider a number of small synthetic transfer challenges,
each representing a controlled transformation between the source and target domains. 
We present two of these challenges. 
Eight additional synthetic examples are available in Appendix~\ref{sec:Synthetic}.

Each synthetic example consists of $1,000$ independent trials. 
The ``moving source'' transformation demonstrates a source concept that is shifted in the target domain, 
but retains its geometry between domains. 
By design, we expect STRUT to excel in this case.
In the ``mixed boxes" transformation, the source concept is composed of 
a 50-50 mixture of slightly shifted boxes and the target concept consists of one of these boxes. 
This problem models a case where the target concept is a kind of refinement of the source concept. 
One can expect SER to excel in this problem.

In Table~\ref{tab:examples-synth-figs} we depict the concepts learned by STRUT, SER and MIX for the two transformations.
The corresponding test errors are presented in Table~\ref{tab:experiments-synth-err}. 
Indeed, in these simple cases the algorithms perform as expected.
The performance of MIX in these examples is clearly not the average performance of SER and STRUT;
in the `moving source' example MIX is a close runner up to the best algorithm, and 
in `mixed boxes' it is better than both STRUT and SER.

From the additional synthetic examples available in Appendix~\ref{sec:Synthetic}, 
we can see that SER obviously outperforms STRUT in cases where feature correspondence is not maintained between source and target, 
such as OCR and domains of pixel based images.
However, STRUT can easily outperform SER when feature correspondence is maintained, 
such as in the case of the inversion problem.

Another observation from the examples in Appendix~\ref{sec:Synthetic} is that MIX can outperform its constituents or at worst be a close second.
Furthermore, when MIX is only the second best algorithm, 
its error rates are not simply the average of both base algorithms but are much closer to the best algorithm. 
While MIX is a simple ensemble of different algorithms, 
it is capable of providing the desired beneficial results.
Further discussion on this behavior and its causes are found in Section~\ref{sec:explain-mix}.

\section{Empirical Evaluation}
\label{sec:experiments}

We evaluated the SER, STRUT and MIX transfer learning algorithms over a number of challenges, 
first comparing our results with non-transfer learning techniques and trivial tree-based transfer learning baselines 
and finally competing against other model transfer algorithms.

% Baselines
We used the \emph{SrcOnly} baseline as our first benchmark. 
Because it represents a trivial approach to transfer learning that utilizes no target data, 
the model was trained using only the source data.
Our second benchmark was the \emph{TgtOnly} baseline, 
where we create the target model using target only data.
In general, any useful transfer learning method should surpass the SrcOnly baseline. 
The TgtOnly benchmark is traditionally viewed as a \emph{skyline}, representing the best possible performance. 
However, effective transfer learning methods can sometimes surpass the skyline due to clever exploitation of both 
source and target examples, thus enjoying a larger training sample than that allotted to TgtOnly. 

% Naive transfer benchmarks
In addition, we compared performance to trivial tree-based model transfer baselines on the same experimental setup.
The \emph{relabeling} classifier simply 
updates the leaves of a forest trained on the source examples using the target samples. 
In the \emph{bias} approach, the weights in the original forest are changed from a
uniform distribution to one which favors trees with lower error rates on the available target training samples.
\emph{Pruning} stands for using the target samples to perform pruning on the original forest,
just like the reduction step in our SER algorithm or the pruning technique of the C4.5 algorithm \cite{quinlan1993c4}.

% ASVM and consensus regularization 
Finally, we compared performance to two well-known model transfer algorithms. 
The first is \emph{Adaptive SVM} (ASVM) \cite{yang2007cross}, 
which uses target examples to regularize an SVM model with a Gaussian kernel, trained using source examples only.
While ASVM has several extensions, 
these usually rely on a large set of unlabeled target training data, 
without which these techniques are similar to ASVM \cite{duan2009domain, duan2012domain}.
The second algorithm is \emph{consensus regularization} \cite{luo2008transfer}, 
which attempts to decrease the classification error by minimizing an entropy based disagreement measure 
among a set of source-only and target-only models. 
While the original paper applied the algorithm with underlying logistic regression models,
we used random forests, which outperformed the logistic regression application.

\subsection{Datasets}

\begin{table}[t]
\centering
\caption{Dataset information \label{tab:datasets}}
\begin{tabular}{l|c|c|c}
\multicolumn{1}{c|}{dataset} & \multicolumn{1}{c|}{DIM$\left( \cX \right)$} & 
\multicolumn{1}{c|}{DIM$\left( \cY \right)$} & \multicolumn{1}{c}{$\left| S^S \right|$}\\
\hline 
mushroom & 22 &  2 & $4,608$ \\
letter & 16 & 26 & $10,822$ \\
wine & 11 & 11 & $4,898$ \\
digits & 64 & 2 & $5,620$ \\
MNIST & 784 & 10 & $2,000$ \\
USPS & 784 &10 & $10,000$ \\
landmines & 9 & 2 & $8,535$ \\
amazon-webcam & 800 & 10 & $1,123$ \\
caltech-webcam  & 800 & 10 & $958$ \\
\hline
activity-user1 & 35 & 5 & $23,690$ \\
activity-user2 & 35 & 5 & $40,405$ \\
activity-user3 & 35 & 5 & $36,111$ \\
activity-user4 & 35 & 5 & $25,171$ \\
activity-user5 & 35 & 5 & $24,920$ \\
activity-user6 & 35 & 5 & $24,481$ \\
\end{tabular}
\end{table}

% About test sets
The effectiveness of transfer learning techniques is of course expected to depend on the degree of relatedness between $\cD_S$ and $\cD_T$.
We generated the source and target sets based on either meaningful splits of existing datasets, 
or on a transformation of a subset of the dataset. 
These approaches are common practice in dataset construction for validating transfer models \cite{dai2007boosting}.
We used the following data sets, whose properties are presented in Table~\ref{tab:datasets}.

% The different datasets
% Edible vs. poisonous mushroom
{\bf Mushroom}: This is a publicly available dataset from the UCI Repository \cite{Bache+Lichman:2013}.
It contains samples of edible and poisonous mushrooms, 
and the value of the \emph{stalk-shape} binary feature is used to partition the dataset into two.
This technique was used by Dai et al. \cite{dai2007boosting}, and the resulting partition is such that the source mushrooms belong to different species than the target mushrooms.

% Letter
{\bf Letter}: Also from the UCI Repository, the letter recognition dataset is partitioned according to the numeric feature \emph{x2bar}, 
by thresholding on its median for each letter. 
This results in source and target distributions of different fonts.

% Wine quality
{\bf Wine}: Publicly available wine quality dataset \cite{cortez2009modeling}, already partitioned into two; 
the source domain consists of white wines and the target domain consists of red wines.

% 6-9 detection
{\bf Digits}\footnote{\url{ http://tx.technion.ac.il/~omerlevy/datasets/}}: 
The Digits dataset consists of images of handwritten digits.
We considered the two binary problems of identifying `6' and `9', each of which is a $180^O$ rotation of the other.

% Inverted
{\bf Inversion}: The task concerns hand-written digit recognition from the MNIST digit database \cite{mnistlecun}.
Source and target domains are generated from the MNIST database as follows.
For the source domain we take 200 images of each digit, sampled uniformly at random, and invert the color of each image. 
The target data consists of images sampled at random without inversion. 
Test data was taken from the MNIST database.
Examples of inverted digits are shown in Table~\ref{tab:mnist-examples}.

% MNIST with lower resolution source and higher resolution target
{\bf Higher Resolution}: This challenge reflects a scenario where we have a lot of source data, 
from a low-resolution camera, and a small amount of target data, obtained with a high-resolution camera. 
The source examples are generated by an averaging process that creates lower resolution images, 
whereby the source image is partitioned into small `super-pixels', each consisting of a disjoint $4 \times 4$ squares of pixels. 
The intensity of each super-pixel element is averaged.
Test data was taken from the MNIST database.
Examples of low resolution digits are shown in Table~\ref{tab:mnist-examples}.

\begin{table}[t]
 \centering
 \caption{Low resolution and inverted versions of digits from the MNIST dataset}
 \label{tab:mnist-examples}
 \begin{tabular}{r|c|c|c|}
  & \multicolumn{3}{c}{\bf DIGIT} \\
  \cline{2-4}
   & 1 & 4 & 7 \\
  \hline
  \hline
  original image & \fbox{\includegraphics[width=0.03\textwidth]{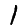}} & 
  \fbox{\includegraphics[width=0.03\textwidth]{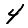}} &
  \fbox{\includegraphics[width=0.03\textwidth]{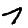}} \\
  \hline
  inverted image & \fbox{\includegraphics[width=0.03\textwidth]{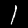}} & 
  \fbox{\includegraphics[width=0.03\textwidth]{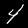}} & 
  \fbox{\includegraphics[width=0.03\textwidth]{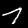}} \\
  \hline
  low resolution & \fbox{\includegraphics[width=0.03\textwidth]{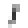}} & 
  \fbox{\includegraphics[width=0.03\textwidth]{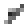}} & 
  \fbox{\includegraphics[width=0.03\textwidth]{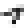}}  \\
%  \hline
 \end{tabular}
\end{table}

% US postal service letter digit detection
{\bf USPS}: Another example of hand-written digit recognition collected under different conditions \cite{hull1994database}. 
The USPS dataset was collected from scans of random letters in a US post office.
We generate the domains using the same transformation as that used in the MNIST database, 
i.e., images are enlarged to 20X20 pixels and placed in a 28X28 image, centered on the center of mass. 
For this transfer experiment we treat MNIST as the source domain and utilize the MNIST training set as source data.

\begin{table*}[t]
\centering
\caption{Test error rates compared to benchmarks and competing algorithms --- lowest error in boldface \label{tab:err-forests-baselines}}
\scalebox{0.85}{
\begin{tabular}{l|ll|lll|ll|lll}
  \multicolumn{1}{c}{\bf DATASET} & \multicolumn{1}{c}{\bf SrcOnly} & \multicolumn{1}{c}{\bf TgtOnly} & \multicolumn{1}{c}{\bf relabeling} & \multicolumn{1}{c}{\bf bias} & \multicolumn{1}{c}{\bf pruning} & \multicolumn{1}{c}{\bf ASVM} & \multicolumn{1}{c}{\bf consensus} & \multicolumn{1}{c}{\bf STRUT} & \multicolumn{1}{c}{\bf SER} & \multicolumn{1}{c}{\bf MIX}\\
  \hline
  mushroom & $15.2\pm0.3$ & $0.5\pm0.1$ & $2.1\pm0.2$ & $12.6\pm0.5$ & $14.1\pm0.6$ & $2.3\pm0.2$ & $0.6\pm0.1$ & $1.9\pm0.2$ & $\mathbf{0.4\pm0.07}$ & $0.5\pm0.08$ \\
  letter & $66.5\pm0.4$ & $19.3\pm0.2$ & $20.7\pm0.3$ & $63.4\pm0.6$ & $22.5\pm0.4$ & $33.8\pm0.2$ & $24.1\pm0.4$ & $21.0\pm0.4$ & $18.9\pm0.2$ & $\mathbf{16.7\pm0.2}$ \\
  wine & $66.6\pm0.6$ & $45.5\pm0.3$ & $44.9\pm0.2$ & $55.3\pm0.2$ & $\mathbf{44.6\pm0.2}$ & $54.7\pm0.4$ & $44.3\pm0.2$ & $46.6\pm0.3$ & $45.8\pm0.2$ & $45.0\pm0.3$ \\
  digits & $19.9\pm0.05$ & $3.0\pm0.2$ & $10.0\pm0.0$ & $19.9\pm0.01$ & $10.0\pm0.0$ & $10.0\pm0.0$ & $14.2\pm0.9$ & $5.4\pm0.3$ & $\mathbf{2.9\pm0.2}$ & $3.8\pm0.3$ \\
  USPS & $13.5\pm0.0$ & $14.9\pm0.2$ & $13.7\pm0.1$ & $13.7\pm0.0$ & $13.3\pm0.1$ & $78.2\pm3.5$ & $\mathbf{11.6\pm0.1}$ & $15.6\pm0.2$ & $13.5\pm0.1$ & $13.3\pm0.1$ \\
  landmines & $52.4\pm0.1$ & $41.0\pm0.7$ & $39.2\pm0.5$ & $52.1\pm0.1$ & $\mathbf{38.2\pm0.2}$ & $43.4\pm0.7$ & $41.6\pm0.7$ & $40.4\pm0.6$ & $40.7\pm0.4$ & $40.4\pm0.5$ \\
  \hline
  amazon-webcam & $62.2\pm0.1$ & $74.6\pm0.8$ & $66.7\pm0.7$ & $65.4\pm0.4$ & $64.6\pm0.7$ & $88.0\pm0.5$ & $67.4\pm0.7$ & $71.3\pm0.6$ & $\mathbf{62.0\pm0.6}$ & $64.6\pm0.7$ \\
  caltech-webcam & $65.1\pm0.4$ & $74.6\pm0.6$ & $66.6\pm0.4$ & $67.6\pm0.4$ & $88.2\pm0.3$ & $68.6\pm0.5$ & $64.3\pm0.4$ & $71.3\pm0.4$ & $\mathbf{63.3\pm0.4}$ & $64.6\pm0.5$ \\
  \hline
  inversion(1\%) & $98.7\pm0.2$ & $54.2\pm0.1$ & $54.8\pm5.5$ & $96.2\pm3.8$ & $57.3\pm5.8$ & $92.4\pm0.3$ & $76.0\pm0.3$ & $\mathbf{41.8\pm0.5}$ & $58.5\pm0.3$ & $44.1\pm0.3$ \\
  inversion(5\%) & $98.7\pm0.2$ & $28.2\pm0.2$ & $36.4\pm3.7$ & $96.9\pm3.1$ & $39.5\pm4.0$ & $92.4\pm0.2$ & $44.8\pm0.3$ & $\mathbf{21.2\pm0.3}$ & $36.5\pm0.2$ & $22.2\pm0.1$ \\
  inversion(10\%) & $98.7\pm0.2$ & $20.5\pm0.1$ & $30.7\pm3.1$ & $97.0\pm3.0$ & $34.3\pm3.5$ & $92.4\pm0.2$ & $32.7\pm0.2$ & $\mathbf{15.9\pm0.4}$ & $24.1\pm0.3$ & $16.4\pm0.2$ \\
  \hline
  high-res(1\%) & $32.7\pm0.2$ & $54.2\pm0.1$ & $44.6\pm4.5$ & $\mathbf{33.7\pm3.4}$ & $42.4\pm4.3$ & $90.2\pm0.3$ & $35.7\pm0.5$ & $48.2\pm0.5$ & $37.2\pm0.3$ & $38.5\pm0.3$ \\
  high-res(5\%) & $32.7\pm0.2$ & $28.2\pm0.2$ & $24.8\pm2.5$ & $32.5\pm3.3$ & $24.4\pm2.5$ & $90.2\pm0.2$ & $23.7\pm0.3$ & $28.3\pm0.3$ & $22.5\pm0.2$ & $\mathbf{21.8\pm0.1}$ \\
  high-res(10\%) & $32.7\pm0.2$ & $20.5\pm0.1$ & $19.4\pm2.0$ & $32.3\pm3.3$ & $20.3\pm2.0$ & $90.2\pm0.2$ & $19.4\pm0.2$ & $22.9\pm0.4$ & $18.0\pm0.3$ & $\mathbf{17.3\pm0.2}$ \\
  \hline
  Activity(min) & $11.5\pm0.2$ & $13.5\pm0.2$ & $11.4\pm0.1$ &$11.8\pm0.02$ & $11.2\pm0.1$ & $76.0\pm0.2$ & $14.4\pm0.3$ & $11.6\pm0.2$ & $11.2\pm0.3$ & $\mathbf{11.1\pm0.2}$ \\
  Activity(median) & $15.1\pm0.3$ & $16.2\pm0.2$ & $14.7\pm0.1$ & $15.6\pm0.03$ & $14.6\pm0.1$ & $76.9\pm0.2$ & $16.8\pm0.1$ & $14.6\pm0.2$ & $13.9\pm0.1$ & $\mathbf{13.8\pm0.3}$ \\
  Activity(max) & $15.4\pm0.1$ & $18.5\pm0.3$ & $14.7\pm0.2$ & $15.1\pm0.04$ & $14.4\pm0.2$ & $74.0\pm0.1$ & $18.1\pm0.2$ & $15.0\pm0.1$ & $15.0\pm0.2$ & $\mathbf{14.2\pm0.2}$
 \end{tabular}
}
\end{table*}

% Landmines detection
{\bf Landmine}\footnote{\url{http://www.ee.duke.edu/~lcarin/LandmineData.zip}}: 
The landmine dataset consists of information collected from 29 real mine fields. 
Each field is represented by 9 features collected using sonar images. 
15 of these fields were dense in foliage, while the other 14 came from barren areas. 
We attempt to use the information from the foliage covered fields to improve mine detection in the barren areas.

% The office-caltec image recognition
{\bf Office-Caltech}\footnote{\url{http://www-scf.usc.edu/~boqinggo/domain_adaptation/GFK_v1.zip}}: 
This dataset is a collection of images of 10 categories from 4 domains. 
We transfer information from larger domains with higher resolution images, 
Amazon.com product pages or the Caltech10 collection, and attempt to recognize lower resolution webcam images.

% Activity Recognition from Maayan
{\bf Activity Recognition}: This dataset, collected by Subramanya et al. \cite{subramanya2012recognizing}, 
is a recording of a customized wearable sensor system. 
The system recorded measurements on 6 users doing various activities, such as walking or running,
going up or down the stairs, or simply lingering. 
We performed the same preprocessing on the data as that performed by Harel and Mannor \cite{harel2011learning} 
and treated each ordered pair of users as a source-target pair, totaling 30 possible pairs.

\subsection{Experiments and Results}

% How we handle most of the data and use the forest
We set $S^S$ to be all available source data, and $S^T$ to be 5\% (unless specified otherwise) of the target samples, 
stratified and randomly selected; the rest was used as test data, 
giving us around $20\left| S^T \right| \cong \left| S^S \right|$ in almost all datasets. 
In all cases the models consisted of 50 decision trees. 
Following Breiman's work on random forest learning, we consider only a log number of features, 
selected at random, when performing feature selection (see also Louppe et al. \cite{NIPS2013_4928}). 

% The two special datasets
The landmine detection and activity recognition tasks exhibit special characteristics. 
Both tasks have class imbalance, where in the landmines problem only 6\% of the examples were positive (mine found), 
and in the activity problem, running and going up or down the stairs totaled less than 10\% of the examples. 
In these experiments we ascertained that the ratio of classes in the training data was similar to that of the entire target dataset. 
Moreover, with the severe class imbalance exhibited, 
error (or accuracy) is no longer an appropriate measurement and can lead to erroneous conclusions \cite{galar2012review}.
Therefore, in these cases we measured the balanced error rate (BER): 
$BER=\frac{1}{c}\Sigma_{i=1}^{c}\frac{e_i}{n_i}$, where $e_i$ and $n_i$ are the number of errors and the number of samples in class $i$ respectively, and $c$ is the number of classes.

% Baselines
We began by comparing the performance of our algorithms and the two benchmarks.
Results of these tests for the SER, STRUT and MIX algorithms are presented in Table~\ref{tab:err-forests-baselines}.
For the ``inversion'' and ``high-res'' datasets, the table includes performance at 1\%, 5\% and 10\% source-target ratios.
For the ``activity recognition" task, the table present the best case, worst case and median cases 
(minimum, maximum and median TGT error, respectively).

% The observed results
SER often, but not always, outperformed STRUT, while MIX produced a classifier that was best, 
or close to best, among all the methods,
even when one of the underlying forests performed poorly as happened, 
for example, in the three ``inversion'' sets. 
Such results indicate that MIX is robust to inferior performance of one of its constituents\footnote{
We ascertained that the good MIX results are not due to
``unfair" model complexity conditions, as each base method contains
50 trees, while MIX is a union of them all (100 trees).
To this end, we repeated all experiments with STRUT and SER containing
100 trees. No significant performance improvements were recorded
due to this modification.}.
We note that MIX continuously outperformed the SrcOnly baseline and in many cases matched or outperformed the TgtOnly baseline.

% Transfer benchmarks
Next, we compared our algorithms to the previously described model transfer baselines and competing algorithms.
We focus our discussion on the MIX algorithm which generally performed well.
Following the results provided in Table~\ref{tab:err-forests-baselines}, 
we note that our MIX algorithm constantly outperformed the relabeling and bias benchmarks. 
Similarly, with the exception of the wine and landmines experiments, 
our MIX algorithm outperformed the simple pruning approach. 
Finally, we ascertained the superiority of our algorithms over the benchmarks using a t-test 
with p-value $<$ 0.005 for the relabeling, bias and pruning benchmarks.
Our algorithms also show success compared to ASVM and consensus regularization. 
These results were ascertained as statistically significant using a t-test (p-value $< 0.005$).

%An interesting counterexample would be the ``mushroom" problem, 
%where consensus regularization greatly outperforms our STRUT algorithm.
%However, the poor performance of STRUT in this case is expected as 
%STRUT is ineffective when the domain does not consist of numeric features, 
%as is the case for the ``mushroom'' set. 

% Learning curves
Figure~\ref{fig:inversion-error} presents the learning curves for the algorithms on the ``inversion" and ``high-res" datasets. 
The curves show error as a function of the ratio between source and target sizes. 
As seen before, our MIX algorithm yielded similar results to the better of the underlying algorithms, 
matching the error rates of SER in the ``high-res" problem and coming close to STRUT in the case of ``inversion".
In both cases the MIX algorithm outperformed the TgtOnly benchmark.

\begin{figure*}[t]
 \label{fig:inversion-error}
 \centering

  \subfloat[Error rates for the inversion problem]{
   \fbox{\includegraphics[width=.47\linewidth]{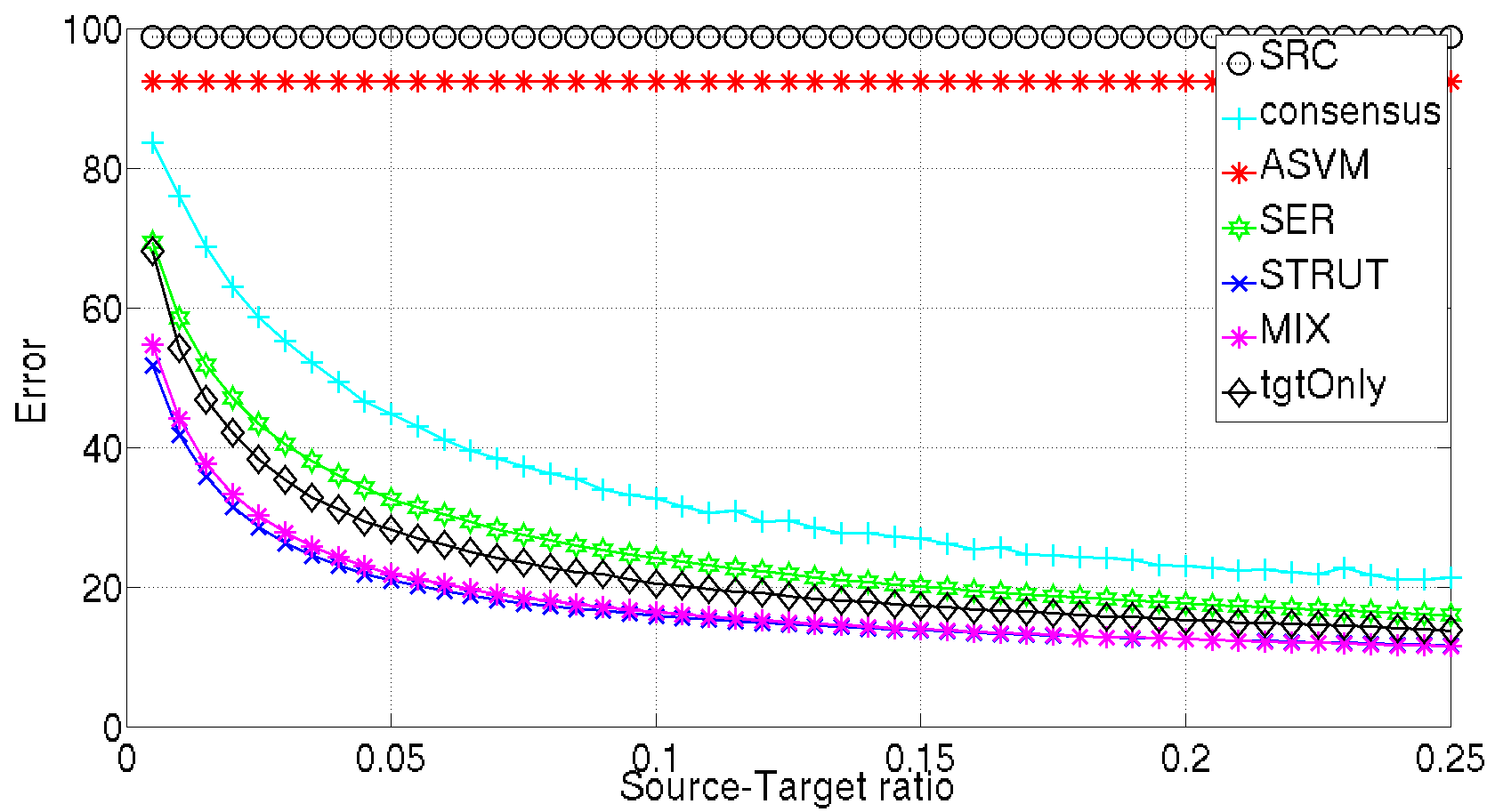}}
  }
  \subfloat[Error rates for the high-res problem]{
   \fbox{\includegraphics[width=.47\linewidth]{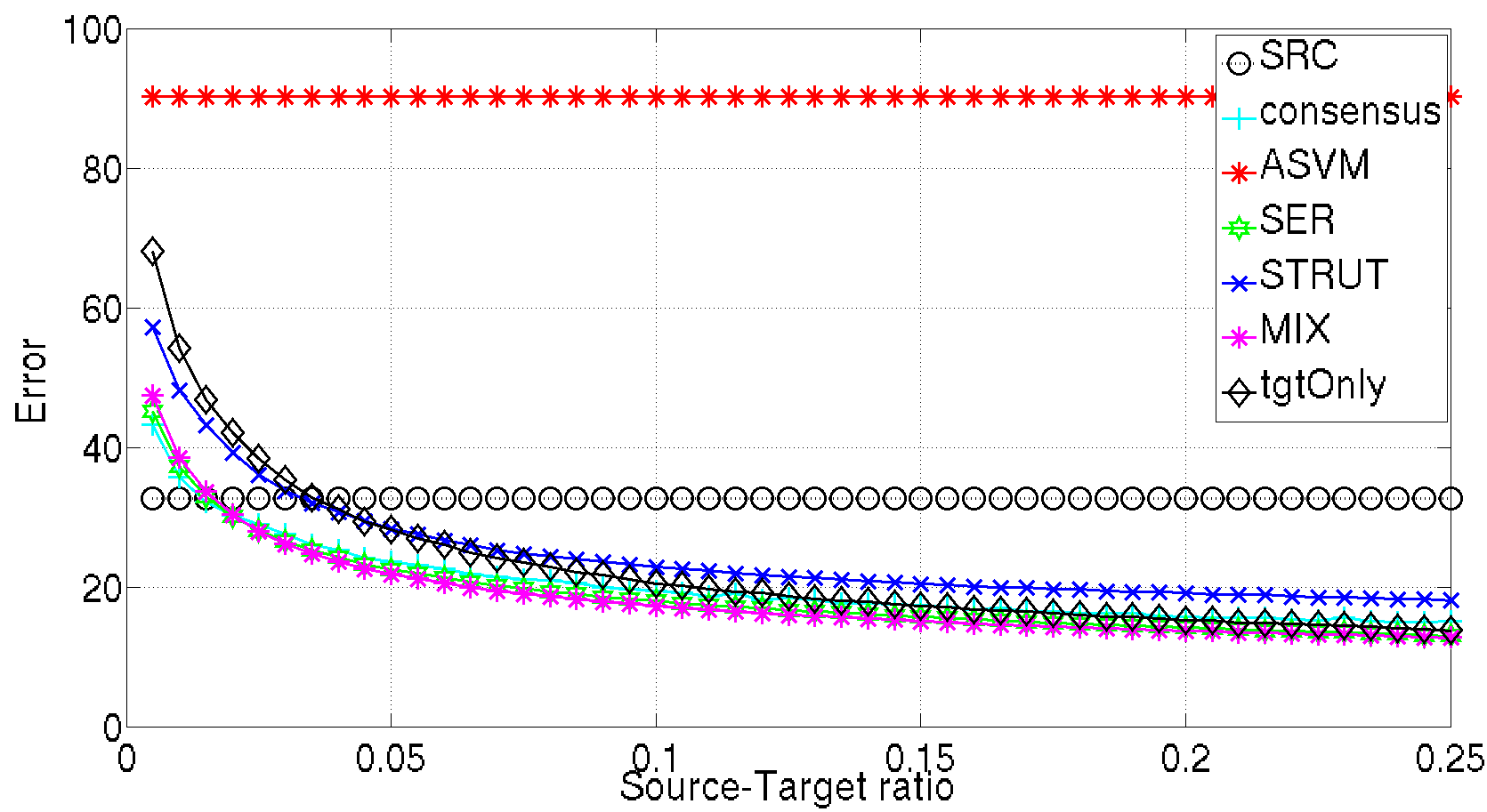}}
  }
 \caption{Error rates on the MNIST problems. The x-axis is the ratio between available source and target training examples. }
\end{figure*}

% Time complexity
Finally, we would like to comment on the time complexity recorded while performing these experiments. 
We note that each tree in the forest can be processed independently, 
allowing for easy parallelism. 
We have observed that the average model transformation time of the ``letter" problem in a serial execution is 3.1s for MIX, 
1.6s for consensus regularization, and 11s for ASVM, while on a 10-core machine we saw linear improvement, 
with MIX taking 0.31s (To the best of our knowledge, there is no parallel version of the consensus regularization algorithm).
This advantage of our techniques is clearly visible in Table~\ref{tab:transfer-time}, 
where the average transfer runtime of our algorithms clearly superior to ASVM transfer.
In today's world of high throughput and massively parallel computing, 
a forest containing dozens or even hundreds of trees can be trained in almost the same time that it takes to build a single tree. 

\begin{table}[t]
\caption{Models transfer times in MS. STRUT and SER times are shown for a serial execution. \label{tab:transfer-time}}
\centering
\begin{tabular}{l|ll|ll}
\multicolumn{1}{c|}{\bf DATASET} & \multicolumn{1}{c}{\bf consensus} & \multicolumn{1}{c}{\bf ASVM} & \multicolumn{1}{c}{\bf STRUT} & \multicolumn{1}{c}{\bf SER}\\
\hline
mushroom & $106.2$ & $15.1$ & $1.8$ & $6.9$ \\
letter & $110.5$ & $7,739.8$ & $224.1$ & $149.2$ \\
wine & $9.3$ & $262.2$ & $77.1$ & $61.8$ \\
digits & $204.9$ & $592.3$ & $32.3$ & $47.5$ \\
landmines & $52.3$ & $50.2$ & $37.7$ & $35.7$ \\
\hline
inversion(1\%) & $7.1$ & $1,1843.5$ & $30.99$ & $174.3$ \\
inversion(5\%) & $23.3$ & $53,685.4$ & $186.1$ & $761.2$ \\
inversion(10\%) & $45.2$ & $122,839.4$ & $396.5$ & $1,327.4$ \\
\hline
high-res(1\%) & $7.9$ & $11,745.4$ & $23.7$ & $104.1$ \\
high-res(5\%) & $26.5$ & $53,819.7$ & $152.4$ & $355.7$ \\
high-res(10\%) &  $52.8$ & --- & $323.2$ & $940.9$ \\
\hline
Activity & $173.9$ & --- & $128.3$ & $37.5$ \\
\end{tabular}
\end{table}

\subsection{Comparing to Instance Transfer Algorithms}
\label{sec:comparingToInstance}
Unlike model transfer, in the \emph{instance transfer} approach to transfer learning,
all source training examples are available during the adaptation to the target.
At the outset, this additional information can lead to better performance.
In this sense, a comparison of a model transfer algorithm that learns without source examples
to an instance transfer algorithm is unfair. 
Nevertheless, it is interesting and important to understand the benefits and limitations of model 
transfer methods, and therefore, we conducted a comparative study of our model transfer 
methods to instance transfer algorithms. 

In this section we briefly mention our comparison of the MIX algorithm versus instance transfer algorithms.
The first is \emph{TradaBoost} \cite{dai2007boosting}, 
which is applied with random decision trees as the weak learners. 
Our tests show that the use of random decision trees produces much better results linear SVMs, 
as suggested by TradaBoost's authors.
The second algorithm tested was \emph{TrBagg} \cite{kamishima2009trbagg},
which initially trains classifiers on bootstrapped bags sampled with replacements from $T^S \cup T^T$ 
and regularizes the ensemble by filtering out classifiers which are overly biased towards the target domain.
TrBagg is also applied with random decision trees and for the filtering phase we use the MVT filtering technique, 
as suggested by the authors.
In all experiments we applied TradaBoost and TrBagg with up to 50 iterations.
The third algorithm we compared against  is the \emph{Frustratingly Easy Domain Adaptation} (FEDA) \cite{daume07frustratingly} meta-algorithm. 
FEDA generates a new middle-ground domain to train on by transferring the data from both source and target to the middle-ground domain.
To compare apples-to-apples we also applied FEDA with a random forest as its underlying
algorithm.
Finally, we test the \emph{Mixed-Entropy} (ME) \cite{goussies2014transfer} algorithm, 
a state-of-the-art forest-specific technique which combines source and target training samples using a weighted information gain measure.
The results of these experiments are presented in Table~\ref{tab:err-forests-instance}.

% The results compared to the competition
Our algorithms routinely outperform most other techniques and are competitive with FEDA.
Surprisingly, our study shows that MIX is comparable and even competitive with instance transfer algorithms, 
despite the unfair comparison.
In particular, MIX often showed similar results to FEDA and Mixed-Entropy 
and consistently outperformed TradaBoost TrBagg. 
For example, the error rates of TradaBoost, TrBagg, FEDA and Mixed-Entropy for the ``letter" dataset were $41.7$, $29.3$, $18.8$ and $17.7$, respectively.
The advantage of MIX over TradaBoost and TrBagg was backed by t-tests with all p-values $ < 0.01$.
No statistically significant performance difference could be observed for FEDA and MIX or ME and MIX.

\begin{table}[t]
\caption{Test error rates of the MIX transfer forest and the instance transfer algorithms --- lowest error in boldface \label{tab:err-forests-instance}}
\centering
\begin{tabular}{l|cccc|c}
\multicolumn{1}{c|}{\bf DATASET} & \multicolumn{1}{c}{\bf TRADA} & \multicolumn{1}{c}{\bf TrBagg} & \multicolumn{1}{c}{\bf FEDA} & \multicolumn{1}{c}{\bf ME} & \multicolumn{1}{c}{\bf MIX}\\
\hline
mushroom & $2.1$ & $\mathbf{0.4}$ & $0.5$ & $\mathbf{0.4}$ & $0.5$ \\
letter & $41.7$ & $29.3$ & $18.8$ & $17.7$ & $\mathbf{16.7}$ \\
wine & $57.6$ & $48.4$ & $\mathbf{43.2}$ & $46.5$ & $45.0$ \\
digits & $19.9$ & $15.3$ & $10.0$ & $\mathbf{2.9}$ & $3.8$ \\
USPS & $33.2$ & $ 14.3 $ & $13.5$ & $13.8$ & $\mathbf{13.3}$ \\
landmines & $45.8$ & $49.1$ & $\mathbf{38.8}$ & $40.5$ & $40.4$ \\
inversion(5\%) & $64.9$ & $34.3$ & $28.2$ & $27.1$ & $\mathbf{22.2}$ \\
high-res(5\%) & $58.4$ & $24.8$ & $\mathbf{20.4}$ & $24.8$ & $21.8$ \\
Activity(median) & $20.3$ & $20.2$ & $15.9$ & $20.2$ & $\mathbf{13.8}$ \\
\end{tabular}
\end{table}

\section{Can We Explain the Advantage of MIX?}
\label{sec:explain-mix}

Our empirical results indicate that the MIX algorithm 
performs well even when just one of its constituents gives good results and can moreover outperform each of its constituents.
We attribute this behavior to diversity and correlation among the ensemble members.
A given tree transformed by the SER algorithm is likely to be different in size than the original tree, 
as the expansion step will add to the tree depth and the reduction step will reduce the size of some of the branches, 
while the same tree transformed by the STRUT algorithm is likely to retain its original size but with different thresholds.
Thus, the pairwise correlation in the MIX forest between two trees transformed from the same original tree 
are expected to exhibit low correlation and result in a more diverse forest.

Let $y f(x)$ denote the classification margin of a soft binary classifier $f$ with respect to a 
point $x$ whose label is $y$, i.e., $yf(x)>0$ iff $f(\cdot)$ is correct on $(x,y)$.
We consider the expected error of an ensemble with weight distribution $Q$ over its members 
via two risk functions commonly used in PAC-Bayesian literature, 
the \emph{Bayes risk} $R \left( B_Q \left( f \right) \right)$, 
also called \emph{risk of the majority vote}, and the \emph{Gibbs risk} $R \left( G_Q \left( f \right) \right)$, 
found in Equations~\ref{eq:bayes-risk} and \ref{eq:gibbs-risk} respectively.
\begin{equation}
\label{eq:bayes-risk}
R \left( B_Q \left( f \right) \right) = \underset{\left( x,y \right) \sim D}{\mathbb{E}}
I \left( \underset{f \sim Q}{\mathbb{E}} y\cdot f \left( x \right) \leq 0 \right),
\end{equation}
\begin{equation}
\label{eq:gibbs-risk}
R \left( G_Q \left( f \right) \right) = \underset{\left( x,y \right) \sim D}{\mathbb{E}}
\left( \underset{f \sim Q}{\mathbb{E}} I \left( f \left( x \right) \neq y \right) \right).
\end{equation}
%It is easy to prove that under a uniform weight distribution within the forest, 
%the Gibbs risk for the MIX forest is the average Gibbs risk of the STRUT and SER forests.
%We do not include this simple proof here.

While it is well known that $R \left( B_Q \right) \leq 2R \left( G_Q \right)$ 
(e.g., \cite{shawe2003pac, mcallester2003simplified, germain2009pac}), 
Germain et al. have shown that with more pairwise ``non-correlated" ensemble members 
(those with a non-positive pairwise covariance of their risk between ensemble members), 
it is possible to provide the tighter bound found in Corollary~\ref{cor:c-bound} 
using the measure of expected disagreement, $d_Q$. 
\begin{corollary}[Corollary 16\cite{germain2015risk}]
\label{cor:c-bound}
Given n voters having non-positive pairwise covariance of their risk under a uniform distribution Q, we have\\
$R \left( B_Q \right) \leq {\cal{C}}_Q  \leq \frac{1}{n\cdot \left( 1-2R \left( G_Q \right) \right)^2}$ \\
where ${\cal{C}}_Q = 1-\frac{\left( 1 - 2 \cdot R \left( G_Q \right) \right)^2}{1 - 2 \cdot d_{Q}}$ \\
and $d_{Q}=\underset{\left( x,y \right) \sim D}{\mathbb{E}}
\left( \underset{f_1 \sim Q}{\mathbb{E}} \underset{f_2 \sim Q}{\mathbb{E}} I \left( f_1 \left( x \right) \neq f_2 \left( x \right) \right) \right)$.
\end{corollary}
Thus, 
in the likely case where two trees transformed from the same original tree are ``non-correlated", 
the bound on the Bayes risk for the MIX forest is nearly halved as $n$ doubles.

\begin{figure}[t]
 \centering
 \includegraphics[width=1\linewidth]{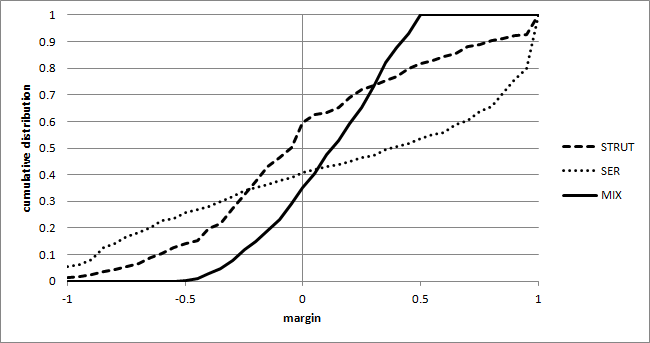}
 \caption{Cumulative distribution functions of margins for disagreeing base classifiers}
 \label{fig:cumulative}
\end{figure}

We now consider the `mixed boxes' experiment presented Section~\ref{sec:examples}, 
where SER is significantly better than STRUT and MIX is even better than SER (Table~\ref{tab:experiments-synth-err}).
We calculated the empirical Gibbs risk of the forests, 
measuring $0.17$ and $0.15$ for the STRUT and SER forests respectively.
Following the last inequality in Corollary~\ref{cor:c-bound} we get bounds of $0.14$ and $0.13$ 
on the Bayes risk of the STRUT and SER forests, respectively, 
and a much tighter bound of $0.07$ for the MIX forest.
These results indicate that the Bayes risk for the MIX algorithm are expected to be lower than those of its constituents, 
as are the actual test results.

Germain et al. have also shown the following formulation to the ${\cal{C}}$-bound:
$$
{\cal{C}}_Q = 1-\frac{\left( 1 - 2 e_Q -d_Q \right)^2}{1 - 2 \cdot d_{Q}}
$$ 
where 
$$
e_Q = \underset{\left( x,y \right) \sim D}{\mathbb{E}}
\left( \underset{f_1 \sim Q}{\mathbb{E}} \underset{f_2 \sim Q}{\mathbb{E}} 
I \left( f_1 \left( x \right) \neq y \right) I \left( f_2 \left( x \right) \neq y \right) \right)
$$ is a measure of expected joint error \cite{germain2015risk}. 
We measured the empirical joint error and disagreement, 
noting that the empirical joint error of the three algorithms was similar 
while the disagreement measure of MIX was much higher than that of SER or STRUT, 
which resulted in lower ${\cal{C}}_Q$ bound for the MIX algorithm.

We also informally argue that the attractive property of the MIX advantage over its constituents is related to the
distribution of empirical pointwise classification margins in cases where SER and STRUT disagree in their predictions.
In Figure~\ref{fig:cumulative} we plot the cumulative distribution functions (CDFs) of empirical margins obtained by SER, 
STRUT and MIX for the same `mixed boxes' experiment when SER and STRUT disagree. 
The advantage of MIX and SER here is evident by their lower CDF values at the origin in Figure.~\ref{fig:cumulative}.
Schapire et. al. addressed these circumstances and related better generalization to better empirical margin profiles, 
as given in Theorem~\ref{thrm:margin-profile}:
\begin{theorem}[Theorem~1\cite{schapire1998boosting}]
\label{thrm:margin-profile}
let ${\cal{D}}$ be a domain over ${\cal{X}} \times \left\{ -1, 1 \right\}$ with distribution , 
let $S$ be a sample of $m$ examples chosen independently at random according to ${P}$.
Assume that the base-classifier space ${\cal{H}}$ is finite and let $\delta > 0$.
Then with probability at least $1 - \delta$ over the random choice of the training set $S$,
every weighted average function $f$ satisfies the following bound for all $\theta > 0$:

$\underset{\left( x,y \right) \sim {\cal{D}}}{p} \left[ y f\left( x\right) \leq 0 \right] \leq$

$\underset{\left( x,y \right) \sim S}{p} \left[ y f\left( x\right) \leq \theta \right] + 
\mathcal{O} \left( \frac{1}{\sqrt{m}} \left( \frac{\log m \log \left| {\cal{H}} \right|}{\theta^2} + 
\log \left( \frac{1}{\delta} \right) \right)^{1/2} \right)$.
\end{theorem}
% Theorem~1 in \cite{schapire1998boosting} addresses these circumstances and relates better generalization to better empirical margin profiles.

Intuitively, when an ensemble algorithm is correct, its underlying classification margins tend to be
high and correlated, and when it is wrong, its underlying 
%negative 
margins tend to be more dispersed as the result of low pairwise correlation. 
Combining STRUT and SER in MIX benefits from some correctly
performing constituents within the erroneous ensemble. 
In other words, as STRUT and SER are only weakly correlated, MIX benefits when combining them. 

\section{Related work}
\label{sec:related}

% The three types of induced transfer learning
The generic title ``transfer learning'' encompasses quite a few different paradigms. 
As noted by Levy and Markovitch\cite{levy2012teaching}, such paradigms are motivated
by (implicit or explicit) modeling or process assumptions.
For example, some paradigms, such as ``feature transfer'', are motivated by assumptions on the linkage between source and target  domains (e.g., features at the target obtained by certain mappings applied on the source features). 
%Others, such as ``model transfer'', are characterized by limitations on the transfer
%process itself (e.g., it is not permissible to transfer training examples from source to target).
The survey by Pan et al. \cite{pan2010survey} identifies the following settings, 
which are not mutually exclusive.

{\bf Model Transfer}: This setting, within which the present work resides, assumes that a good predictor for the source has been learned, 
resulting in an attempt to adapt the model to the target problem using a training set from the target domain. 
Model transfer techniques are effective when a similar inductive bias performs well for the related tasks 
or when source examples are impossible to retain or distribute. 
Present model transfer model methods rely on a biased regularizer 
\cite{kienzle2006personalized, yang2007cross, rodner2008learning, tommasi2010safety, rodner2011learning}, 
on aggregating multiple source-target predictors 
\cite{baxter2000model, rettinger2006boosting, luo2008transfer, ruckert2008kernel}, 
utilizing model parameter transfer as priors \cite{pratt1991direct, thrun1994learning, eaton2008modeling}, 
or by feature weight estimation \cite{eruhimov2008transferring, rodner2009learning}.

{\bf Instance Transfer}: In this setting 
one assumes certain instances of the source data can be used as examples in the target domain. 
Under this assumption, it is better to take some of the source data ``as is", 
and the problem reduces to identifying the relevant instances and ignoring the irrelevant ones, 
using a process of elimination or weighting. 
Boosting based \emph{instance weighing} is common practice in this category \cite{dai2007boosting, pardoe2010boosting, yao2010boosting, lu2013selective}, 
as is \emph{instance elimination} (and sub-sampling) \cite{jiang2007instance, kamishima2009trbagg},
but other techniques exist for utilizing the source information in different ways \cite{wu2004improving, daume07frustratingly, saenko2010adapting, goussies2014transfer}. 

{\bf Features Transfer}: Assuming some partial relation between the source and target features exists, 
algorithms working in this setting attempt to learn a feature mapping or weighing. 
These techniques represent an attempt to find the ``common denominators" of the learning tasks, 
matching features, or combinations of features, to identify meaning in partial information. 
Standard techniques to address this problem include norm optimization \cite{raina2007self, evgeniou2007multi, harel2011learning} 
and manipulating and combining features \cite{eaton2008modeling, pan2008transfer, levy2012teaching}

{\bf Domain Adaptation (DA)} and {\bf Multi-Task Learning (MTL)}:
In domain adaptation the difference between the domains is the result of different feature and labeling spaces; 
however, DA is typically considered within a semi-supervised context where an abundance of unlabeled data is available as well \cite{daume2006domain, ben2010theory, jiang2008literature},
while in MTL the goal is to produce a good hypothesis for several related learning problems simultaneously \cite{caruana1998multitask}. 
%Typically, the different tasks are defined on the same input space but the probability distributions differ. 
Some notable approaches for these settings are based on similarity to a common predictor
\cite{evgeniou2004regularized,dredze2010multi,daume07frustratingly, jie2011multiclass},
finding a shared representation \cite{blitzer2006domain, jiang2008cross, saenko2010adapting, duan2012exploiting, fernando2013unsupervised} or a shared subspace
\cite{argyriou2008convex,ando2005framework,lounici2009taking, baktashmotlagh2013unsupervised}, 
as well as probabilistic approaches \cite{lawrence2004learning,schwaighofer2005hierarchical,yu2005learning,bonilla2007multi}. 

{\bf Semi-supervised Transfer}: source and target domains are the same and source data includes only unlabeled examples \cite{liao2005logistic, raina2007self, jiang2008cross, duan2009domain, duan2009domain-b, duan2012domain, duan2012domain-b, duan2012visual}.
Semi-supervised transfer is very attractive in application areas such as machine vision, video event detection and text analysis. 

For a comprehensive review of these fields
the reader is referred to the works of Pan and Yang\cite{pan2010survey} and Jiang\cite{jiang2008literature}. 

\subsection{Transfer Learning using Decision Trees}
Most early transfer learning methods were based on neural networks, 
and while SVMs and ensemble techniques have become prominent in this field, 
DT models are still under-explored in this setting. 
Of the few tree-based techniques researched, 
only the work by Won et al. operates in our model transfer setting.
Specifically, Won et al. proposed a simple technique to update an existing tree trained only on source samples using target samples \cite{won2007transfer}.
Their approach resembled a batch iterative learning technique which relies solely on iterative expansion steps.
This technique does not consider any refitting of numeric feature thresholds.

%Most early transfer learning methods were based on neural networks, 
%and while SVMs and ensemble techniques have become prominent in this field, 
%DT models are still under-explored in this setting. 
%Of the few tree-based techniques researched, 
%we are aware of only two which operate in our model transfer setting.

%Specifically, Won et al. proposed a simple technique to update an existing tree trained only on source samples using target samples \cite{won2007transfer}.
%Their approach resembled a batch iterative learning technique which relies solely on iterative expansion steps.
%This technique does not consider any refitting of numeric feature thresholds.
%Eruhimov et al. focused on using a pre-existing tree ensemble to estimate feature weights, 
%allowing for the construction of new trees based solely on target samples using the learned weight bias \cite{eruhimov2008transferring}.

Adaptive DTs and stream sub-sampling have been used in data streams to handle massive, high-speed streams \cite{domingos2000mining, jin2003efficient, hulten2005mining} 
as well as concept drift \cite{hulten2001mining, gama2006decision, nunez2007learning}, 
modifying the DT as new samples arrive. 
However, there are no explicit source and target distributions in this setting, 
but a distribution that incrementally changes over time, 
requiring retaining of original samples and constant modifications to the DTs.

%Similarly, incremental tree induction has been suggested as a way to update trees over time as data is collected. 
%Various methods have been suggested, e.g., 
%probabilistic techniques \cite{gehrke1999boat}, pruning extensions\cite{hapfelmeier2014pruning} or fuzzy rules \cite{domingos2000mining,  last2002online, cohen2008info},
%however, these algorithms still assume that train and test samples are collected from the same underlying distribution, 
%making them very different than our own model transfer techniques.

Finally, in MTL, Faddoul et al. presented a variation of AdaBoost with decision stumps fitted to multiple tasks \cite{faddoul2010boosting}. 
The same authors later applied boosting with DTs while using a modified information gain (IG) criterion  \cite{faddoul2012learning}. 
Another approach builds an ensemble by combining multiple random DTs, 
where task-driven splits are added in each tree, in addition to ordinary feature splits \cite{simm2014treeMTL}. 
In this manner, the trees may contain branches uniquely dedicated to particular subsets of tasks.

\section{Concluding Remarks}

Exploiting the modularity and flexibility of decision trees, 
we designed two model transfer learning algorithms that utilize a model trained over the source domain and effectively adapted it to 
the target domain using local adjustments of the tree parameters and/or its architecture. 
Our experiments with synthetic data indicate that the effectiveness of the algorithms varies with the transformation type. 
Our final MIX algorithm combines the proposed base algorithms and often outperform their underlying constituents. 
It also achieve performance superior to leading model transfer algorithms and, 
moreover, are competitive with  instance transfer algorithms and even outperform some of them.

An attractive feature of the proposed method (and any effective model transfer algorithm) is that the source data can be discarded after training over the source domain, 
and transferring the models to the target domain can be computed later on, whenever needed. 
A nice application of this property would be to devise a bank of models computed over a variety of source domains that 
can later be used to construct models for any related target domain.

An open issue is to capture formally and systematically the ramifications of possible 
source/target transformations over the tree structure. 
For example, we asserted above that some geometrical transformation of the support of the inputs density 
can be captured and modeled via threshold changes in a decision tree. 
Yet it would be interesting to formally map and relate these and other data transformations to the tree adaptation mechanisms.

Furthermore, our work has not touched upon sample complexity and formal generalization ability. 
This problem of devising a comprehensive learning theory for transfer learning of decision trees 
might be contingent upon formally defining an effective model for possible relations between the source and the target.

\appendices
\section{Other Numerical Examples on Synthetic Data}
\label{sec:Synthetic}

In Section~\ref{sec:examples} we presented two small synthetic transfer learning challenges to illustrate the behavior of our algorithms.
These two examples were selected from a set of 10 problems we synthetically designed to capture simple transformations between the source and target problems. 
Here we define all 10 of these problems, 
present the test performance of the algorithms over these problems, 
and provide details on the experimental protocol used.

Each challenge is defined via a fixed binary concept over the source domain and a transformation that maps the
concept to the target domain. The concepts and transformations are defined below for each challenge.
All challenges were defined such that $\cX$ is the positive unit quadrant in $\reals^3$ (using numeric features).
In all experiments we maintained the relation $\left| S^S \right| = 5 \left| S^T \right|$ and took  $|S^T| = 64$. 
In all cases, $P(x)$, the marginal distribution over $\cX$, is the uniform distribution over $\cX$.
Each challenge was randomly repeated $1,000$ times and each test error reported was 
computed as the averages over a test set consisting of $10,000$ random target domain points.
This large number of trials ensured statistically significant comparisons of the results.

\begin{table}[t]
\centering
\begin{tabular}{l|c|c|ccc}
 & \multicolumn{1}{c}{$\cD_S$}  & \multicolumn{1}{c}{$\cD_T$}  & \multicolumn{1}{|c}{\bf STRUT}  & \multicolumn{1}{c}{\bf SER}  & \multicolumn{1}{c}{\bf MIX} \\
 \hline
 mixture & \fbox{\includegraphics[width=0.025\textwidth]{mixSource2.png}} &  \fbox{\includegraphics[width=0.025\textwidth]{mixTarget.png}} & $7.7$ & $6.6$ & $\mathbf{5.5}$ \\
 \hline
 inversion & \fbox{\includegraphics[width=0.025\textwidth]{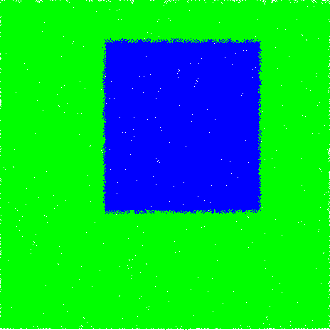}} &  \fbox{\includegraphics[width=0.025\textwidth]{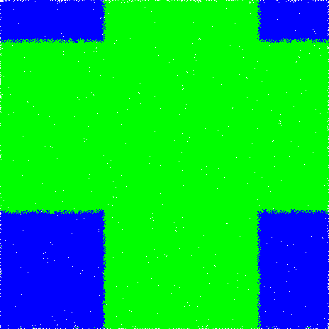}} & 6.1 & $\mathbf{6.0}$ & $\mathbf{6.0}$ \\
 \hline 
 moving source & \fbox{\includegraphics[width=0.025\textwidth]{source.png}} & \fbox{\includegraphics[width=0.025\textwidth]{movingTarget.png}} & $\mathbf{6.1}$ & $12.8$ & $6.4$ \\
\hline
expanding & \fbox{\includegraphics[width=0.025\textwidth]{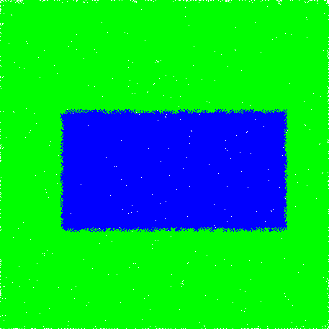}} & \fbox{\includegraphics[width=0.025\textwidth]{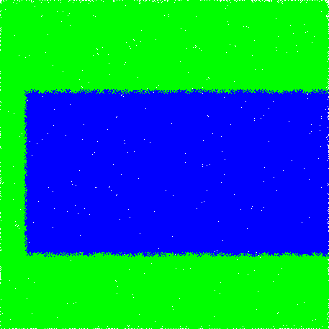}} & 10.7 & 15.4 & $\mathbf{10.6}$ \\
\hline
shrinking & \fbox{\includegraphics[width=0.025\textwidth]{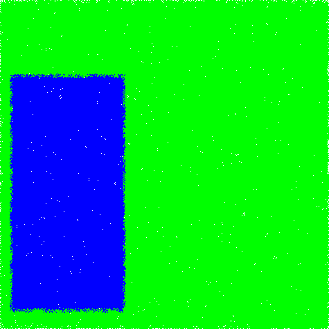}} & \fbox{\includegraphics[width=0.025\textwidth]{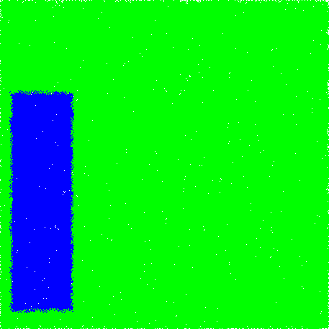}} & 7.1 & $\mathbf{3.7}$ & 4.9 \\
\hline
axis swap & \fbox{\includegraphics[width=0.025\textwidth]{source1.png}} & \fbox{\includegraphics[width=0.025\textwidth]{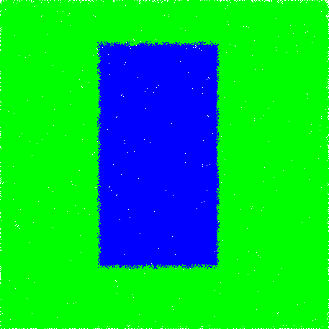}} & 9.0 & 11.8 & $\mathbf{8.9}$ \\
\hline
noisy target & \fbox{\includegraphics[width=0.025\textwidth]{source1.png}} & \fbox{\includegraphics[width=0.025\textwidth]{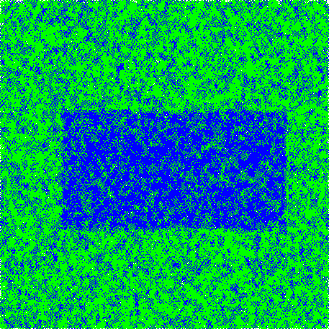}} & 23.4 & $\mathbf{14.9}$ & 15.1 \\
\hline
noisy source & \fbox{\includegraphics[width=0.025\textwidth]{noisySource1.png}} & \fbox{\includegraphics[width=0.025\textwidth]{source1.png}} & 18.8 & $\mathbf{6.6}$ & 8.2 \\
\hline 
rotated source & \fbox{\includegraphics[width=0.025\textwidth]{source.png}} & \fbox{\includegraphics[width=0.025\textwidth]{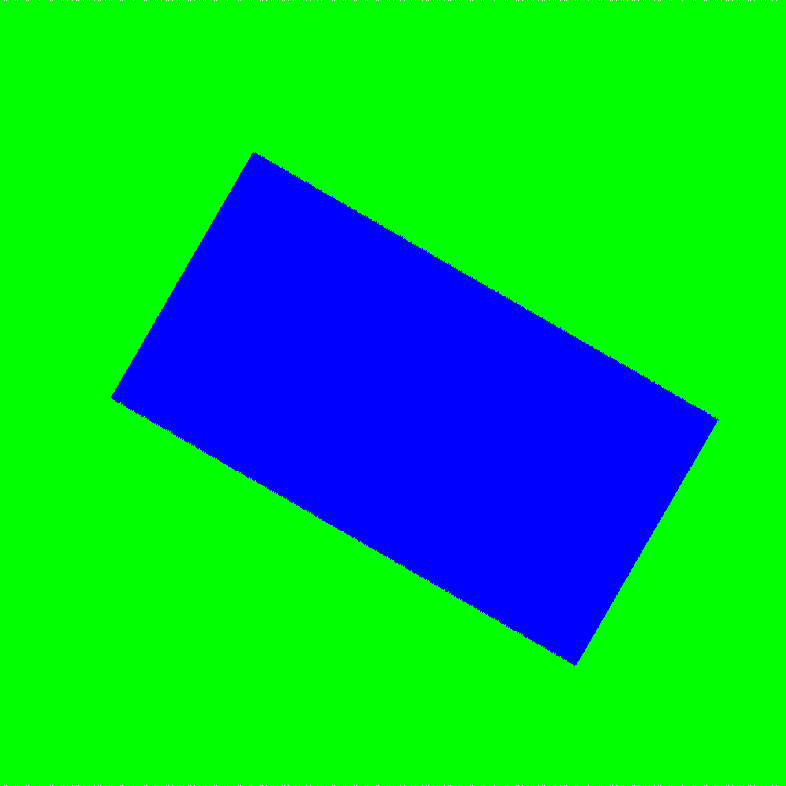}} & 19.2 & $\mathbf{13.4}$ & 16.3 \\
\hline
fish-eye & \fbox{\includegraphics[width=0.025\textwidth]{source1.png}} & \fbox{\includegraphics[width=0.025\textwidth]{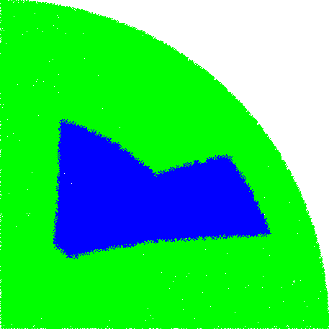}} & 18.5 & $\mathbf{17.3}$ & 17.8 \\
\hline
refined sine & \fbox{\includegraphics[width=0.025\textwidth]{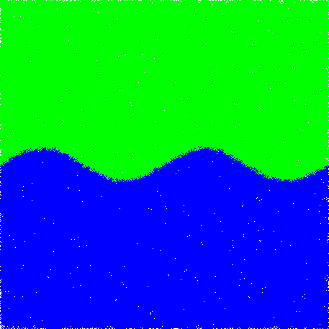}} & \fbox{\includegraphics[width=0.025\textwidth]{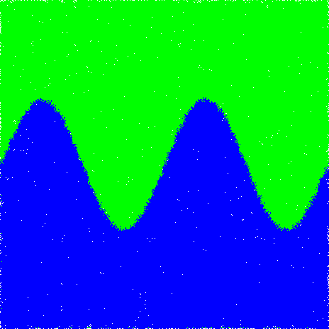}} & $\mathbf{13.3}$ & 14.3 & 13.8 \\
\end{tabular}
\caption{Error measurements for transfer forests on synthetic source-target transformations\label{tab:experiments-synth-figs}}
\end{table}

Each challenge is defined via a binary ``concept'' and a transformation. The concept is the source region where points are labeled '$+$'.
In most cases the concept is a random 3D box in the source domain, $\cX_S$, whose volume is 25\% of vol$(\cX_S)$.
Below we refer to such a box as a ''standard random box''.
We define the following transformations:
\begin{enumerate}[\labelwidth=0.2em]
\item
In the \emph{source/target mix}, the source concept is a 50-50 mixture of two standard random boxes,
and the target concept is a single random box among the two. 
\item 
In \emph{source inversion} the source concept is a standard random box. We represent this box by  
$(\alpha_0,\alpha_1,\alpha_2)$ and $(\beta_0,\beta_1,\beta_2)$, such that a point $x=(x_0,x_1,x_2)$ is labeled as positive in the source domain 
(i.e., it is in the box) iff $ \alpha_i\leq x_i \leq \beta_i$, $i = 0, 1, 2$.
Then , the target concept is defined such that a target point, $x=(x_0,x_1,x_2)$, is labeled positive iff 
$x_i <\alpha_i$  or $\beta_i < x_i$, $i = 0, 1, 2$.
\item
In \emph{moving source} the source concept is a standard random box 
and the target concept is obtained by a random displacement of the source box along each of the axes that still keeps the displaced 
box within $\cX_T$.
\item
In the \emph{expanding} and \emph{shrinking} source challenges, 
the source concept is a standard random box. In the expanding challenge we expand the source box so that its volume is 
doubled in the target domain, and in the shrinking challenge its volume is halved. 
%The expansion and shrinkage are obtained by randomly selecting a triplet, $s_0, s_1, s_2$, with $s_i$ being the expansion/shrinkage factors
%along the $i$th dimension, and the product of these factors is precisely 2 in the expansion challenge, and $1/2$ in the shrinking challenge.
\item
In \emph{axis-swap} the source concept is a standard random box and the target concept is obtained by swapping two randomly selected dimensions (among the three).
\item
In \emph{noisy target} the source concept is a standard random box and the target concept is the same box but each 
%point in $\cX_T$ inverts its label with probability 0.25.
point inverts its label with probability 0.25.
\item 
The \emph{noisy source} challenge is precisely the inverse of the noisy target challenge where the target concept is a
standard random box and the source concept is its noisy distortion.
\item
In the \emph{rotated source} challenge, a standard random box in the source is rotated about a random vector by a random angle using 
a standard 3D rotation matrix.
\item
In the \emph{fish-eye transformation} challenge, each point is represented within a spherical coordinate system, 
namely, $x=\left( r, \theta, \phi \right)$ where $0 \leq \theta \leq \frac{\pi}{4}$ and $0 \leq \phi \leq \frac{\pi}{2}$. 
For every source point $x_s=\left( r_s, \theta_s, \phi_s \right)$, there exists a maximum value $r_m$ such that 
$\left( r_m, \theta_s, \phi_s \right)$ is still inside the source feature space $\cX_S$; 
thus, we transfer the point $x_s$ from the source domain to the point $x_t=\left( r_t, \theta_s, \phi_s \right)$ in the target domain, 
such that $r_t = \frac{r_s}{r_m}$. This is of course a deterministic transformation applied to each standard random concept (box) in the source.
\item
Our final challenge is the \emph{refined sine boundary}, in which the source concept is defined by a sine wave manifold that changes 
frequency and amplitude in the target.
A source point $x=\{x_0,x_1,x_2\}$ is labeled positive iff $0.5 + 0.05 \cdot \sin\left(4\pi\left( x_0 + x_1\right) \right) < x_2$.
Our target domain is defined similarly, but with random frequency $0.25 \leq \phi \leq 0.5$,
and a random amplitude $0 \leq a \leq 0.5$.
Thus, a point $x= (x_0,x_1,x_2)$ in the target domain is labeled positive iff 
$0.5 + a \cdot \sin\left(\frac{2\pi}{\phi}\left( x_0 + x_1\right) \right) < x_2$.
\end{enumerate}
In Table~\ref{tab:experiments-synth-figs} we present the test error results of STRUT, SER and MIX for these challenges.

% you can choose not to have a title for an appendix
% if you want by leaving the argument blank
%\section{}
%Appendix two text goes here.

%% use section* for acknowledgment
%\ifCLASSOPTIONcompsoc
%  % The Computer Society usually uses the plural form
%  \section*{Acknowledgments}
%\else
%  % regular IEEE prefers the singular form
%  \section*{Acknowledgment}
%\fi
%
%
%The authors would like to thank...
%

% Can use something like this to put references on a page
% by themselves when using endfloat and the captionsoff option.
\ifCLASSOPTIONcaptionsoff
  \newpage
\fi

% trigger a \newpage just before the given reference
% number - used to balance the columns on the last page
% adjust value as needed - may need to be readjusted if
% the document is modified later
%\IEEEtriggeratref{8}
% The "triggered" command can be changed if desired:
%\IEEEtriggercmd{\enlargethispage{-5in}}

% references section

% can use a bibliography generated by BibTeX as a .bbl file
% BibTeX documentation can be easily obtained at:
% http://www.ctan.org/tex-archive/biblio/bibtex/contrib/doc/
% The IEEEtran BibTeX style support page is at:
% http://www.michaelshell.org/tex/ieeetran/bibtex/
%\bibliographystyle{IEEEtran}
% argument is your BibTeX string definitions and bibliography database(s)
%\bibliography{IEEEabrv,../bib/paper}

\bibliographystyle{IEEEtran}
% argument is your BibTeX string definitions and bibliography database(s)
\bibliography{IEEEabrv,trees_transfer}

%
% <OR> manually copy in the resultant .bbl file
% set second argument of \begin to the number of references
% (used to reserve space for the reference number labels box)

%\begin{thebibliography}{1}
%
%\bibitem{IEEEhowto:kopka}
%H.~Kopka and P.~W. Daly, \emph{A Guide to \LaTeX}, 3rd~ed.\hskip 1em plus
%  0.5em minus 0.4em\relax Harlow, England: Addison-Wesley, 1999.
%
%\end{thebibliography}

% biography section
% 
% If you have an EPS/PDF photo (graphicx package needed) extra braces are
% needed around the contents of the optional argument to biography to prevent
% the LaTeX parser from getting confused when it sees the complicated
% \includegraphics command within an optional argument. (You could create
% your own custom macro containing the \includegraphics command to make things
% simpler here.)
%\begin{IEEEbiography}[{\includegraphics[width=1in,height=1.25in,clip,keepaspectratio]{mshell}}]{Michael Shell}
% or if you just want to reserve a space for a photo:

%\begin{IEEEbiography}{Michael Shell}
%Biography text here.
%\end{IEEEbiography}

% if you will not have a photo at all:

\vspace{-1 cm}
\begin{IEEEbiographynophoto}{Noam Segev}
Received his B.Sc. degree in software engineering from the Technion-Israel Institute of Technology, 
graduating with honor in 2010. 
His research interests include machine learning and pattern recognition.
\end{IEEEbiographynophoto}

\vspace{-2 cm}
\begin{IEEEbiographynophoto}{Maayan Harel}
Received her B.Sc. degree in biomedical engineering from Tel-Aviv University, in 2009, 
and her Ph.D. degree in electrical engineering from the Technion-Israel Institute of Technology, in 2015.
Her research interests include machine learning, pattern recognition and statistics.
\end{IEEEbiographynophoto}

% insert where needed to balance the two columns on the last page with
% biographies
%\newpage

\vspace{-2 cm}
\begin{IEEEbiographynophoto}{Shie Mannor}
Received his B.Sc. degree in electrical engineering, B.A. degree in mathematics, and Ph.D. degree in electrical
engineering from the Technion-Israel Institute of Technology, in 1996, 1996, and 2002, respectively.
From 2002 to 2004, he was a Fulbright scholar and a postdoctoral associate at M.I.T. 
He was with the Department of Electrical and Computer Engineering at McGill University from 2004 to 2010 
where he was the Canada Research chair in Machine Learning.
He has been with the Faculty of Electrical Engineering at the Technion since 2008 where he is currently a professor.
His research interests include machine learning and pattern recognition,
planning and control, multi-agent systems, and communications. 
\end{IEEEbiographynophoto}

\vspace{-2 cm}
\begin{IEEEbiographynophoto}{Koby Crammer}
Received his Ph.D. (2004) in Computer science and BSc (1999) in mathematics, 
physics and computer science all from the Hebrew university Jerusalem and all Summa cum laude. 
He was a postdoctoral fellow and a research associate at the Department of Computer and Information Science, 
University of Pennsylvania between from 2004 to 2009.
Koby joined the Department of Electrical Engineering at the Technion-Israel Institute of Technology since 2009.
Koby published more than fifteen journal papers and sixty conference papers, 
member of the editorial board of machine learning journal and journal of machine learning research, 
served in the program committee of leading conferences in machine learning and organized few international workshops.
His research interests are in the design, analysis and study of machine learning and recognition methods and 
their application to real world data and especially natural language processing, and big-complex data sets. 
Recipient of the Rothschild fellowship, Fulbright fellowship (declined), Alon fellowship and Horev fellowship.
\end{IEEEbiographynophoto}

\vspace{-2 cm}
\begin{IEEEbiographynophoto}{Ran El-Yaniv}
Received his Ph.D. in theoretical computer science from Toronto University.
He is an Associate Professor of Computer Science at the Technion-Israel Institute of Technology, 
and  his research interests include machine learning, online computation and computational finance.
Ran is a co-author of the book “Online Computation and Competitive 
Analysis” (Cambridge university Press, 1998), and member of the 
editorial boards of the Journals of AI Research and Machine Learning Research.
\end{IEEEbiographynophoto}

% You can push biographies down or up by placing
% a \vfill before or after them. The appropriate
% use of \vfill depends on what kind of text is
% on the last page and whether or not the columns
% are being equalized.

%\vfill

% Can be used to pull up biographies so that the bottom of the last one
% is flush with the other column.
%\enlargethispage{-5in}

% that's all folks
\end{document}